\documentclass{article}

    \PassOptionsToPackage{numbers, compress}{natbib}
 \usepackage[main, final]{neurips_2026}
\usepackage[utf8]{inputenc} % allow utf-8 input
\usepackage[T1]{fontenc}    % use 8-bit T1 fonts
\usepackage{hyperref}       % hyperlinks
\usepackage{url}            % simple URL typesetting
\usepackage{booktabs}       % professional-quality tables
\usepackage{amsfonts}       % blackboard math symbols
\usepackage{nicefrac}       % compact symbols for 1/2, etc.
\usepackage{microtype}      % microtypography
\usepackage{xcolor}         % colors

\definecolor{best}{RGB}{244,181,180}    % 最优（红粉）
\definecolor{second}{RGB}{249,219,184}  % 次优（浅橙）
\definecolor{third}{RGB}{253,244,197}   % 次次优（浅黄）

\usepackage{amsmath}
\usepackage{wrapfig}
\usepackage{multirow}
\usepackage[dvipsnames, table]{xcolor}
\usepackage{makecell}
\usepackage{subcaption}
\usepackage{adjustbox}
\usepackage{pifont}
\usepackage{subcaption} 

\usepackage[algo2e,ruled,vlined]{algorithm2e}
\usepackage[most]{tcolorbox}
\usepackage{graphicx}

\usepackage{amsthm}
\newtheorem{definition}{Definition}
\newtheorem{lemma}{Lemma}
\newtheorem{theorem}{Theorem}
\title{SemMSA: Latent Semantic-Aided Robust Multimodal Sentiment Analysis with Incomplete Data}

\author{
Wenhao Li\textsuperscript{1,2},
Zhibin Wu\textsuperscript{1},
Chong Xiao\textsuperscript{1},
Qiangchang Wang\textsuperscript{1}\thanks{Corresponding author.}\! \\
\textsuperscript{1}School of Software, Shandong University \quad
\textsuperscript{2}Shenzhen Loop Area Institute\\
}

\begin{document}

\maketitle

\begin{abstract}
Recent research on Multimodal Sentiment Analysis (MSA) has focused on learning from language, visual, and acoustic modalities with incomplete data to infer human sentiment. Most studies typically compensate for missing information by reconstructing modality features or designing complicated fusion mechanisms. However, these methods still suffer from spurious generation and noisy guidance due to the lack of high-level semantic grounding in partially observed multimodal evidence. To address these issues, we propose SemMSA, a latent semantic-aided framework that constructs rich sentiment-relevant semantics with LLMs, fully integrating with all modalities via anchor-free spectral alignment. It mainly consists of Cross-modal Semantic Refinement (CSR) and Cross-modal Spectral Alignment (CSA). Specifically, CSR first adaptively extracts visual and acoustic representations by corresponding adapters to form a unified multimodal prefix with language in the frozen LLM embedding space. It then iteratively produces continuous discriminative semantic states through a token-efficient latent refinement process without decoding explicit text. Next, CSA simultaneously aligns the refined semantics with all modalities by enhancing the dominant spectral component of their kernel Gram matrix. This captures global nonlinear dependencies among all representations without relying on a predefined anchor modality. In addition, an instance-level spectral separation constraint preserves cross-sample discriminability and mitigates representation collapse.  Extensive experiments on SIMS, MOSI, and MOSEI benchmarks demonstrate that SemMSA achieves state-of-the-art performance. Project page: \url{https://peacelwh.github.io/SemMSA/}.
\end{abstract}
\vspace{-0.15in}

\begin{figure}[t]
  \centering
  \includegraphics[width=0.85\linewidth]{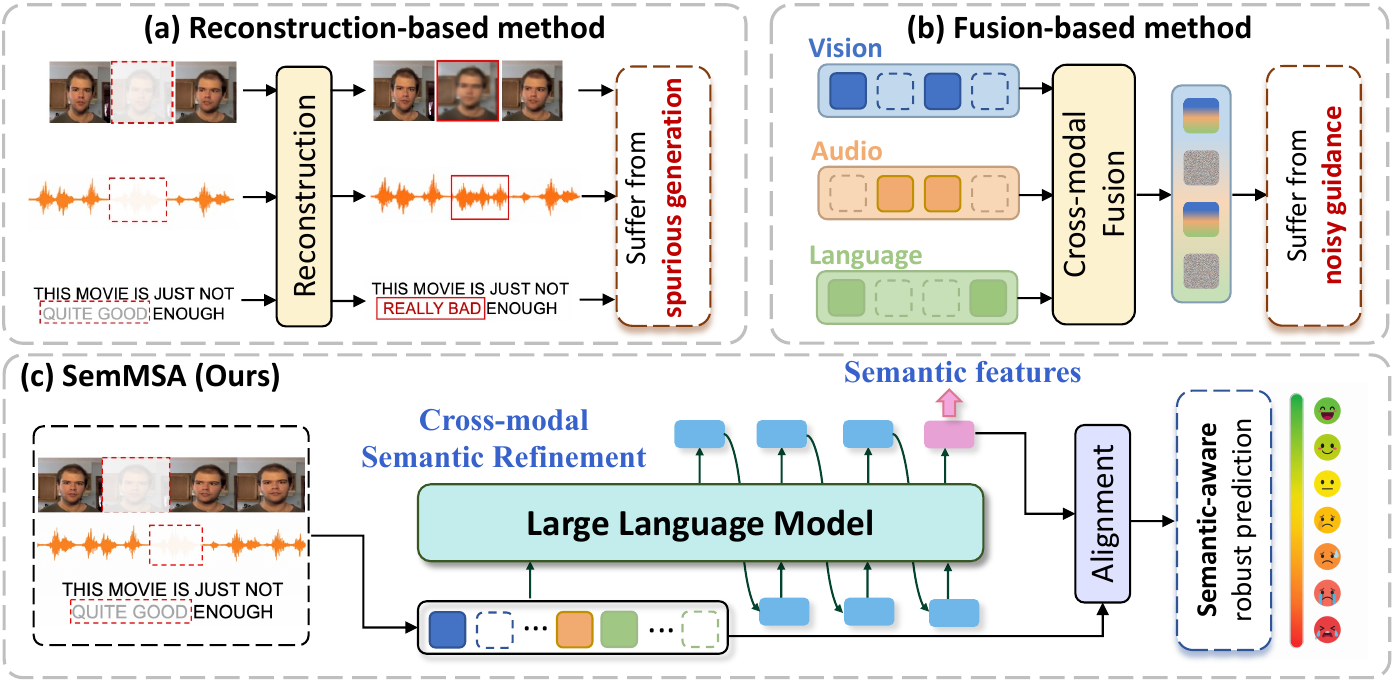} 
  \caption{
Framework comparison between (a) reconstruction-based methods, (b) fusion-based methods, and (c) our SemMSA. Compared with existing methods, SemMSA refines and fully integrates sentiment-relevant semantic information with the LLM to compensate for incomplete data.
}
  \label{fig1}
  \vspace{-0.2cm}
\end{figure}

\section{Introduction}
Multimodal Sentiment Analysis (MSA) aims to learn a comprehensive understanding of human sentiment by modeling multimodal information from language, vision, and audio~\cite{fang2025emoe, wu2025enriching}. Recent research has increasingly moved from controlled laboratory settings toward real-world multimodal scenarios and deployment~\cite{tsai2019multimodal,zhang2023learning,wang2025dlf, xu2026enhancing}. However, multimodal observations are often incomplete due to noise, occlusion, sensor failure, or transmission instability~\cite{li2024unified, zhang2024towards, zhu2025proxy, li2025semantic}. In particular, intra-modal missingness, where partial tokens, frames, or acoustic segments are corrupted within each modality, frequently occurs in practice and significantly degrades models trained under complete data assumptions.

Recent studies have made significant progress, following two paradigms, as shown in Figure~\ref{fig1}(a) and (b). Reconstruction-based methods~\cite{yuan2021transformer, sun2023efficient, zeng2022tag, wang2023incomplete, pu2026f, xu2025irpo} aim to recover complete features from partially observed inputs, alleviating information loss. However, they tend to reconstruct low-level feature patterns rather than high-level semantics that are critical for sentiment understanding. Moreover, the same text may correspond to multiple plausible vocal prosodies and facial expressions, resulting in hallucinating features that contradict the actual sentiment evidence. Fusion-based methods~\cite{zhang2024towards, zhu2025proxy, li2024unified,li2024correlation,kim2024missing, wu2026expanding} instead directly integrate available modalities through complex cross-modal interaction modules to learn robust representations. However, the fused features still suffer from noisy observations and limited evidence. They may overfit modality co-occurrence patterns observed during training under severe missing-modality scenarios, leading to unstable guidance and degraded discrimination. 

Human sentiment understanding often relies on reliable contextual cues and prior knowledge about emotional expression to interpret intentions, attitudes, and affective states under incomplete observations~\cite{camacho2023large, ortega2025integration}.  This motivates semantic-level compensation, where the goal is to enrich incomplete multimodal representations with rich sentiment-relevant semantic information.  Recent studies also indicate a broader shift from discriminative emotion recognition toward generative emotion understanding~\cite{lian2026mer2026discriminativeemotion}. Large Language Models (LLMs)~\cite{liu2026openclawbench, wu2024surgbox, zhang2026stabilizing, xu2026hcsg, yang2026skillmaster, xia2026scoring, yang2026enterprisebench, li2024knn}, with strong contextual representation and semantic abstraction capabilities, provide a promising source of such high-level semantics. However, directly decoding textual descriptions or rationales introduces additional costly overhead. Therefore, an efficient mechanism is needed to exploit LLM-derived semantics in the hidden space.

Inspired by this perspective, we propose SemMSA, a latent semantic-aided framework with LLMs, as shown in Figure~\ref{fig1}(c). First, Cross-modal Semantic Refinement (CSR) bridges heterogeneous non-language modalities and the frozen LLM embedding space through lightweight adapters. Each adapter uses learnable prompts with self- and cross-attention to adaptively aggregate reliable evidence from incomplete visual or acoustic sequences. The produced compact visual and acoustic prefix tokens are concatenated with language embedded in the frozen LLM space. CSR then recurrently appends the last hidden state of the LLM as a continuous latent token, yielding a sequence of sentiment-relevant semantic states. This provides token-efficient semantic compensation without explicit text decoding. Next, Cross-modal Spectral Alignment (CSA) establishes a kernel Gram matrix over the normalized semantic, language, visual, and acoustic representations for comprehensive cross-modal consistency alignment. It encourages all features to align simultaneously along a shared leading direction by enhancing the dominant spectral component of this matrix. This captures global nonlinear dependencies across representations, avoiding the instability from anchor modality under severe missingness. In addition, an instance-wise spectral separation constraint is imposed on the dominant eigenvectors to preserve cross-sample discriminability, mitigating representation collapse. 

Overall, our contributions are summarized as follows:
\begin{itemize}
    \item A SemMSA framework is proposed to provide high-level semantic compensation with LLMs and jointly align heterogeneous representations via an anchor-free spectral objective.
    \item CSR module adaptively aggregates multimodal data into the frozen LLM space and iteratively refines sentiment-relevant semantic states in a token-efficient latent process.
    \item CSA module aligns all representations simultaneously by enhancing the dominant spectral component they form, capturing global nonlinear relationships without anchor dependency.
    \item The proposed method consistently achieves state-of-the-art performance on MOSI, MOSEI, and SIMS under diverse missingness, significantly improving accuracy by 1.4\% on average.
 
\end{itemize}
\section{Related Work}
\subsection{MSA with Incomplete Data}
Most MSA methods often assume that language, visual, and acoustic modalities are fully observed~\cite{zadeh2017tensor, tsai2019multimodal, mai2019locally, lv2021progressive, yang2022disentangled, zhang2023learning, zhangimproving}. They learn unified representations by modeling intra- and inter-modal contextual dependencies. Once modalities are missing or corrupted, models trained under complete settings often degrade sharply. Recent fusion-based methods attempt to alleviate missing information through complex interaction and fusion, including multi-view correlation learning~\cite{andrew2013deep, wang2015deep}, recurrent cross-modal translation~\cite{pham2019found, zhao2021missing}, graph-based dependency modeling~\cite{lian2023gcnet}, and modality-conditioned fusion~\cite{liu2024modality}. Several dominance-guided approaches, such as LNLN~\cite{zhang2024towards} and P-RMF~\cite{zhu2025proxy}, further rely on high-quality primary modalities to anchor multimodal learning. Nevertheless, these methods may become unreliable under severe missingness~\cite{wang2023incomplete}. The scarce representative features with noise in limited data may fail to capture sufficient discriminative information. Another line of work reconstructs missing modalities from observed ones, including cascaded residual auto-encoders~\cite{tran2017missing}, cycle-consistency-based recovery~\cite{pham2019found, zhao2021missing}, graph-based reconstruction~\cite{lian2023gcnet}, low-level feature restoration~\cite{sun2023efficient}, and diffusion-based distribution recovery~\cite{wang2023incomplete}. However, reconstruction-based methods mainly recover superficial statistical patterns rather than sentiment cues essential for sentiment judgment, while diffusion models also incur substantial computational costs. In contrast, SemMSA complements incomplete data with compact LLM-derived latent semantics and comprehensively integrates all representations through a consistent spectral alignment. To the best of our knowledge, this is the first attempt to utilize LLMs to perform semantic-level compensation for MSA.

\subsection{Multimodal Alignment}
Most existing methods~\cite{li2022blip, zhu2023prompt, zhai2023sigmoid, meng2024improving, qi2026global, qi2025cross, li2026dvla} adopt traditional alignment or CLIP-based pairwise contrastive learning~\cite{radford2021learning} . This paradigm has inspired a series of works extending to additional modalities, such as audio-to-text~\cite{elizalde2023clap}, point cloud-to-text~\cite{zhang2022pointclip}, and video-to-text~\cite{ma2022x}. However, aligning each representation only to a single anchor neglects the interactions among the remaining points, making it difficult to capture global structural relationships. This inherently ignores the complexity and richness of multiple modalities. The Gram matrix, which characterizes the mutual geometry among sets of vectors, has shown promise in theoretical analyses of deep learning networks~\cite{pennington2017nonlinear,  li2025vtfsl} and various downstream tasks~\cite{nejjar2023dare, cicchettigramian, liu2026principled, wang2026cast, liang2026multimodal, wang2026navigation}. In contrast, SemMSA introduces a kernelized Gram matrix to capture nonlinear relationships, aligning all representations without anchor dependency.

\section{Method}

We begin by introducing the preprocessing of MSA, and then present two proposed components: (1) how cross-modal latent semantics are generated through CSR, and (2) how all representations are simultaneously aligned via CSA. Figure~\ref{fig2} shows the overview of the proposed SemMSA.

\begin{figure*}[t]
\centering
\includegraphics[width=1\linewidth]{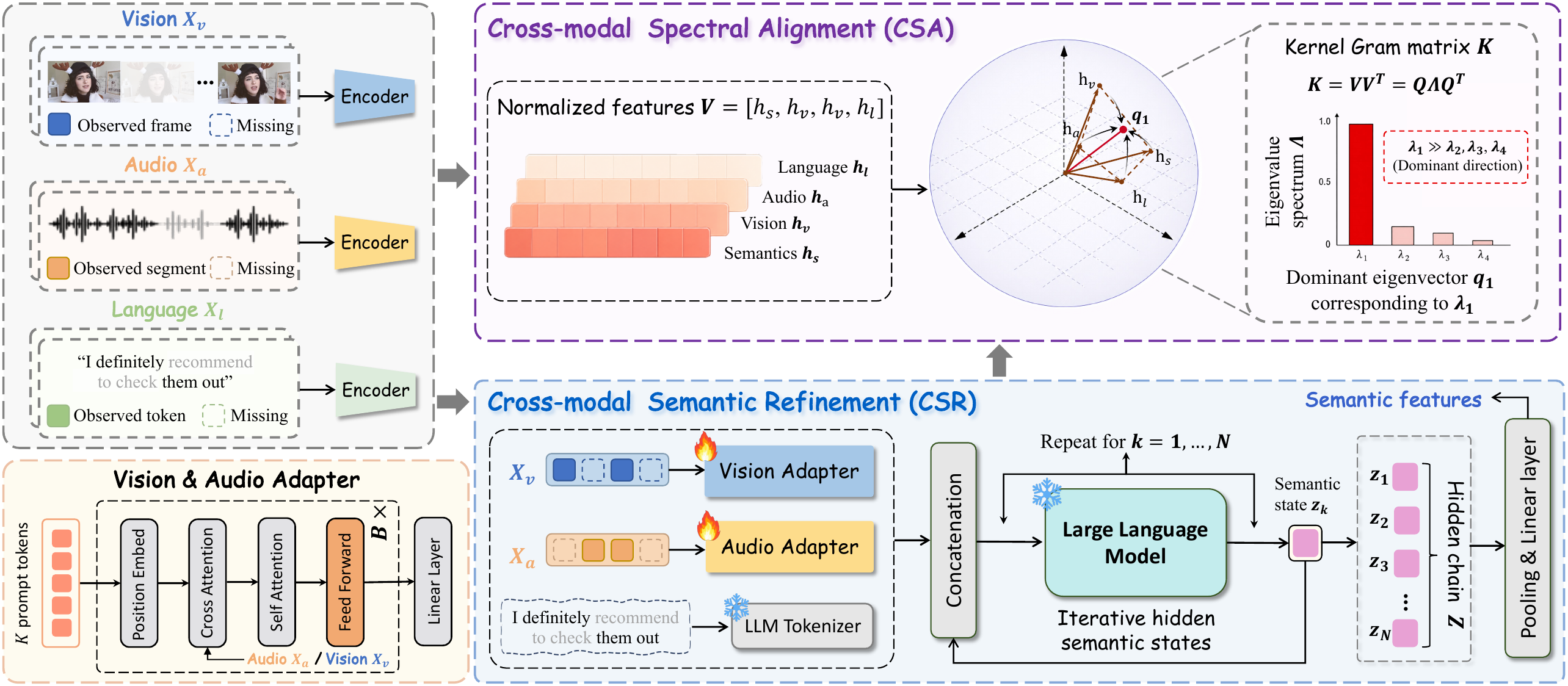}
\caption{Overview of the proposed SemMSA framework. Multimodal inputs with incomplete data are first extracted by encoders. Then, CSR maps visual and audio features via adapters and concatenates with language embedded in the frozen LLM space. Iterative hidden-state refinement is performed to produce compact sentiment-aware latent semantic states. Next, CSA jointly aligns the semantics with visual, acoustic, and language representations by constructing a kernel Gram matrix over normalized features and enhancing its dominant spectral component. }  
	\label{fig2}
 \vspace{-0.3cm}
\end{figure*}

\subsection{Multimodal Input}
Given a multimodal sample for MSA, we consider three modalities of vision, audio, and language. Each modality with incomplete data is denoted by $\mathbf{U}_{m}$, where $m \in \{v,a,l\}$. Following prior works~\cite{zhang2024towards, zhu2025proxy, yuan2021transformer}, each modality is first processed by a widely used modality-specific encoder $f_\phi$ with frozen parameters~\cite{devlin2019bert, mcfee2015librosa, baltruvsaitis2016openface} to obtain a sequence representation $\mathbf{X}_m$ like other domains~\cite{wang2026lanediffusion, wang2025lanediffusion, he2026prior, he2025re}:
\begin{equation}
\mathbf{X}_m = f_\phi(\mathbf{U}_{m}), \quad \mathbf{X}_m \in \mathbb{R}^{T_m \times d_m},
\label{eq1}
\end{equation}
where $T_m$ is the sequence length and $d_m$ denotes the feature dimension of each embedding vector.

\subsection{Cross-modal Semantic Refinement}
\label{csr}

\paragraph{Visual and Audio Adapter.} To extract heterogeneous nonverbal features into the embedding space of a frozen LLM, visual and acoustic adapters with $M$ learnable prompt embeddings $\mathbf{P}_{m} \in \mathbb{R}^{M \times d_m}$ are introduced and updated by a lightweight transformer with $B$ blocks. Due to the features $\mathbf{X}_m$ of frozen encoders without temporal information, learnable positional embeddings $\mathbf{t}_{m} \in \mathbb{R}^{T_m \times d_m}$ are injected into the frame or audio segment representations to obtain  temporal order. $\mathbf{P}_{m}$ first act as queries to perform cross-attention with $\mathbf{X}_m$ to extract modality evidence. These embeddings then interact with each other to capture intra-modality dependencies through self-attention. 

After $B$ blocks, the output of the adapter consists of $M$ modality-specific vectors, one per prompt embedding. They are linearly projected into the dimension of the LLM to obtain $\mathbf{Z}_{v}$, $\mathbf{Z}_{a} \in \mathbb{R}^{M \times d}$. The adapters are optimized to effectively bridge the output of the frozen encoders to the frozen LLM during training. In addition, the language $\mathbf{Z}_{l} \in \mathbb{R}^{T \times d}$ is directly embedded by the LLM, where $T$ is the length of the textual embedding sequence and $d$ is the LLM embedding dimension.

\paragraph{Semantic Refinement.}
The multimodal input $\mathbf{U}^{(0)} = \left[\mathbf{Z}_{l}; \mathbf{Z}_{v}; \mathbf{Z}_{a} \right] \in \mathbb{R}^{(T + 2M) \times d}$ to the frozen LLM is formed by concatenating the three modalities in a fixed order. $\mathbf{U}^{(0)}$ serves as the initial multimodal prefix, providing the frozen LLM with partially observed language, visual, and acoustic evidence in a unified embedding space. Instead of decoding explicit textual descriptions, CSR performs iterative hidden-state refinement directly in the continuous LLM space. Let $\mathcal{F}_{\theta}(\mathbf{U}) \in \mathbb{R}^{|\mathbf{U}| \times d}$ denote the last-layer hidden states produced by the frozen LLM with parameters $\theta$ for an input sequence $\mathbf{U}$. The hidden state at the final position is defined as:
\begin{equation}
\mathbf{f}(\mathbf{U}) = \mathcal{F}_{\theta}(\mathbf{U})_{|\mathbf{U}|} \in \mathbb{R}^{d},
\label{eq2}
\end{equation}
where $\mathbf{f}(\mathbf{U})$ provides a contextual summary of the current multimodal prefix. Based on this representation, the $k$-th latent semantic state is constructed as
\begin{equation}
\mathbf{z}_{k} = \mathbf{f}\!\left(\mathbf{U}^{(k-1)}\right), 
\quad
\mathbf{U}^{(k)} = \left[\mathbf{U}^{(k-1)}; \mathbf{z}_{k}\right],
\quad k=1,\dots,O,
\label{eq3}
\end{equation}
where $O$ is the number of refinement steps, $\mathbf{z}_{k} \in \mathbb{R}^{d}$ denotes the continuous latent semantic state at step $k$, and $\mathbf{U}^{(k)} \in \mathbb{R}^{(T+2M+k)\times d}$ is the updated prefix after appending $\mathbf{z}_k$. Since $\mathbf{z}_{k}$ has the same dimensionality as the LLM token embeddings, it can be appended as a continuous input token for the next refinement step. We use this recurrent operation as a hidden-space semantic refinement mechanism, rather than explicit textual generation. The resulting sequence of latent semantic states is
\begin{equation}
\mathbf{Z} = \left[\mathbf{z}_{1}; \mathbf{z}_{2}; \cdots; \mathbf{z}_{O}\right] \in \mathbb{R}^{O \times d}.
\label{eq4}
\end{equation}
These states provide a compact representation of sentiment-relevant information conditioned on the available multimodal evidence. By repeatedly applying the frozen LLM transformation to the updated prefix, CSR refines the hidden context while avoiding autoregressive decoding of natural-language descriptions. Finally, a compact sentiment-aware semantic representation $\mathbf{H}_{\mathrm{s}}$ is obtained for final fusion and prediction by applying pooling followed by a linear or two-layer MLP projection $g_{\mathrm{s}}$:
\begin{equation}
\mathbf{H}_{\mathrm{s}}
=
g_{\mathrm{s}}\big(\mathrm{Pool}(\mathbf{Z})\big),
\qquad
\mathbf{H}_{\mathrm{s}}\in\mathbb{R}^{d'} .
\label{eq5}
\end{equation}

\subsection{Cross-modal Spectral Alignment}

To achieve comprehensive consistency alignment among semantics, vision, audio, and language, Cross-modal Spectral Alignment (CSA) is proposed to improve global and nonlinear interactions among all representations. Unlike traditional pairwise contrastive objectives which align modality pairs through a predefined anchor, CSA measures the alignment by analyzing the spectral structure of their shared representation matrix. A stronger dominant spectral component indicates that different modalities are concentrated along a common latent direction.

Specifically, three modality representations $\mathbf{X}_m$ are first extracted and unified by their respective modality encoders to obtain $\mathbf{H}_m$ with the ${d'}$ dimension. Each encoder consists of a linear transformation layer followed by Transformer encoder layers. Next, they are transformed into normalized quadruples and construct the matrix $\mathbf{V} = (\mathbf{h}_s, \mathbf{h}_v, \mathbf{h}_a, \mathbf{h}_l)$. The corresponding Gram matrix $\mathbf{G}\in\mathbb{R}^{4\times4}$ is first defined to reflect the pairwise similarity as follows:
\begin{equation}
\mathbf{G}(\mathbf{V})
=
% \mathbf{H}^{\top}\mathbf{H}
\mathbf{V}^\top \mathbf{V}
=
\begin{bmatrix}
\langle\mathbf{h}_s,\mathbf{h}_s\rangle & \langle\mathbf{h}_s,\mathbf{h}_v\rangle & \langle\mathbf{h}_s,\mathbf{h}_a\rangle & \langle\mathbf{h}_s,\mathbf{h}_l\rangle\\
\langle\mathbf{h}_v,\mathbf{h}_s\rangle & \langle\mathbf{h}_v,\mathbf{h}_v\rangle & \langle\mathbf{h}_v,\mathbf{h}_a\rangle & \langle\mathbf{h}_v,\mathbf{h}_l\rangle\\
\langle\mathbf{h}_a,\mathbf{h}_s\rangle & \langle\mathbf{h}_a,\mathbf{h}_v\rangle & \langle\mathbf{h}_a,\mathbf{h}_a\rangle & \langle\mathbf{h}_a,\mathbf{h}_l\rangle\\
\langle\mathbf{h}_l,\mathbf{h}_s\rangle & \langle\mathbf{h}_l,\mathbf{h}_v\rangle & \langle\mathbf{h}_l,\mathbf{h}_a\rangle & \langle\mathbf{h}_l,\mathbf{h}_l\rangle
\end{bmatrix},
\quad
\mathbf{G}_{ij}=\langle\mathbf{h}_i,\mathbf{h}_j\rangle .
\label{eq6}
\end{equation}

To further learn nonlinear cross-modal dependency, this metric is extended into a high-dimensional Reproducing Kernel Hilbert Space (RKHS) using a Radial Basis Function (RBF) kernel mapping $\kappa(\cdot,\cdot)$ The corresponding kernel Gram matrix $\mathbf{K}(V)$ is defined as:
\begin{equation}
\mathbf{K}_{ij}=\kappa(\mathbf{h}_i,\mathbf{h}_j),
\qquad
\kappa(\mathbf{h}_i,\mathbf{h}_j)
=
\exp\left(
-\frac{\|\mathbf{h}_i-\mathbf{h}_j\|_2^2}{2\sigma^2}
\right).
\end{equation}

If the four representations of the same instance are sufficiently aligned in the kernel space, then the kernel Gram matrix should exhibit a significant low-rank structure. In the ideal case, all representations are mapped to the same shared semantic direction, and $\mathbf{K}$ degenerates into a rank-one matrix. Therefore, we perform eigendecomposition on $\mathbf{K}$ as follows:
\begin{equation}
\mathbf{K}=\mathbf{Q}\mathbf{\Lambda}\mathbf{Q}^{\top},
\qquad
\mathbf{\Lambda}=\mathrm{diag}(\lambda_{1},\lambda_{2},\lambda_{3},\lambda_{4}),
\label{eq8}
\end{equation}
where $\lambda_{1}\geq\lambda_{2}\geq\lambda_{3}\geq\lambda_{4}\geq0$. The largest eigenvalue $\lambda_{1}$ represents the principal nonlinear semantic direction shared by the four representations. A higher proportion of it indicates that the representations are more concentrated in the same latent alignment subspace. Therefore, we regard the eigenvalues as logits and enhance the dominance of $\lambda_{1}$ through a softmax-based spectral objective:
\begin{equation}
\mathcal{L}_{\mathrm{csa}}
=
-\frac{1}{N}
\sum_{i=1}^{N}
\log
\frac{\exp(\lambda_{1}^{i}/\tau)}
{\sum_{j=1}^{4}\exp(\lambda_{j}^{i}/\tau)},
\label{eq9}
\end{equation}
where $N$ is the batch size and $\tau$ denotes the temperature parameter.  To prevent degenerate solutions where different instances collapse to the same semantic point, we introduce an instance-wise spectral separation regularizer. 
For the $i$-th instance, the dominant eigenvector $\mathbf{q}_{1}^{i}$ associated with the largest eigenvalue $\lambda_{1}^{i}$ of the kernel Gram matrix indicates the principal alignment pattern among the semantic, visual, acoustic, and language representations. 
Since $\mathbf{q}_{1}^{i}$ only reflects representation-wise combination coefficients, we project it back to the high-dimensional representation space to obtain the dominant semantic direction $\mathbf{u}_{1}^{i}$. 
The separation loss is then defined as:
\begin{equation}
\mathbf{u}_{1}^{i}
=
\frac{\mathbf{V}^{i}\mathbf{q}_{1}^{i}}
{\left\|\mathbf{V}^{i}\mathbf{q}_{1}^{i}\right\|_{2}},
\qquad
\mathcal{L}_{\mathrm{sep}}
=
\frac{1}{N(N-1)}
\sum_{i=1}^{N}
\sum_{\substack{j=1 \\ j\neq i}}^{N}
\left[
(\mathbf{u}_{1}^{i})^{\top}\mathbf{u}_{1}^{j}
\right]^{2}.
\label{eq10}
\end{equation}
where $\mathbf{V}^{i}$ contains the normalized representations of the $i$-th instance.
This regularizer penalizes high similarity between dominant semantic directions of different instances, preserving inter-instance discriminability. These aligned embeddings are fused via element-wise summation to form the final multimodal representation $\hat{\mathbf{H}}$ as follows:
\begin{equation}
\hat{\mathbf{H}} = \mathbf{H}_s + \mathbf{H}_v + \mathbf{H}_a + \mathbf{H}_l.
\label{eq11}
\end{equation}

The fused feature $\hat{\mathbf{H}}$ is finally fed into a linear or MLP regressor $\mathcal{C}$ for sentiment prediction $\hat{y} = \mathcal{C}(\hat{\mathbf{H}})$. The overall loss uses a regression loss $\ell$ (MSE or L1) on $\hat{y}$ and on the semantic prediction $\mathcal{C}(\mathbf{H}_s)$ with weight $\lambda_s$:
\begin{equation}
\mathcal{L}_{\text{total}} = \frac{1}{N}\sum_{i=1}^{N}\Big[\ell\big(y^{i},\hat{y}^{i}\big)+\lambda_s\,\ell\big(y^{i},\mathcal{C}(\mathbf{H}_s^{i})\big)\Big] +
\mathcal{L}_{\mathrm{csa}}
+
\mathcal{L}_{\mathrm{sep}}.
\label{eq12}
\end{equation}

\begin{table*}[t]
\caption{The average results on MOSI and MOSEI datasets over missing rates from 0.0 to 0.9.}
\label{sota1}
\begin{center}
\begin{small}
\resizebox{\textwidth}{!}{
\begin{tabular}{lcccccc cccccc}
\toprule
& \multicolumn{6}{c}{MOSI} & \multicolumn{6}{c}{MOSEI} \\
\cmidrule(lr){2-7} \cmidrule(lr){8-13}
Method & Acc-2 & F1 & Acc-5 & Acc-7 & MAE $\downarrow$ & Corr $\uparrow$
       & Acc-2 & F1 & Acc-5 & Acc-7 & MAE $\downarrow$ & Corr $\uparrow$ \\
\midrule
MISA~\cite{hazarika2020misa}
& 70.33/71.49 & 70.00/71.28 & 33.08 & 29.85 & 1.085 & 0.524
& 75.82/71.27 & 68.73/63.85 & 39.39 & 40.84 & 0.780 & 0.503 \\

Self-MM~\cite{yu2021learning}
& 69.26/70.51 & 67.54/66.60 & 34.67 & 29.55 & 1.070 & 0.512
& 77.42/73.89 & 72.31/68.92 & 45.38 & 44.70 & 0.695 & 0.498 \\

MMIM~\cite{han2021improving}
& 67.06/69.14 & 64.04/66.65 & 33.77 & 31.30 & 1.077 & 0.507
& 75.89/73.32 & 70.32/68.72 & 41.74 & 40.75 & 0.739 & 0.489 \\

TETFN~\cite{wang2023tetfn}
& 67.68/69.76 & 63.29/65.69 & 34.34 & 30.30 & 1.087 & 0.507
& - & - & - & - & - & - \\

TFR-Net~\cite{yuan2021transformer}
& 66.35/68.15 & 60.06/61.73 & 34.67 & 29.54 & 1.200 & 0.459
& 77.23/73.62 & 71.99/68.80 & 34.67 & \cellcolor{second}46.83 & 0.697 & 0.489 \\

ALMT~\cite{zhang2023learning}
& 68.39/70.40 & 71.80/72.57 & 33.42 & 30.30 & 1.083 & 0.498
& 77.54/76.64 & 78.03/77.14 & 41.64 & 40.92 & 0.674 & 0.481 \\

LNLN~\cite{zhang2024towards}
& 70.94/72.55 & 71.25/72.73 & \cellcolor{third}38.27 & \cellcolor{second}34.26 & 1.046 & \cellcolor{third}0.527
& 78.19/76.30 & \cellcolor{third}79.95/77.77 & \cellcolor{third}46.17 & 45.42 & 0.692 & 0.530 \\

P-RMF~\cite{zhu2025proxy}
& \cellcolor{third}71.53/72.81 & \cellcolor{third}71.69/72.93 & \cellcolor{second}38.50 & \cellcolor{third}34.19 & \cellcolor{third}1.038 & 0.525
& \cellcolor{second}78.83/78.14 & \cellcolor{second}80.39/79.33 & 45.87 & 44.63 & \cellcolor{second}0.658 & \cellcolor{second}0.589 \\

TF-Mamba~\cite{li2025tf}
& \cellcolor{second}73.46/72.54 & \cellcolor{second}73.59/72.57 & 37.74 & 33.95 & \cellcolor{second}1.035 & \cellcolor{second}0.548
& \cellcolor{third}77.34/77.61 & 77.18/77.43 & \cellcolor{second}46.64 & \cellcolor{third}45.66 & \cellcolor{third}0.673 & \cellcolor{third}0.578 \\

\hline
SemMSA
& \cellcolor{best}\textbf{74.36/73.91} & \cellcolor{best}\textbf{74.18/73.82} & \cellcolor{best}\textbf{40.93} & \cellcolor{best}\textbf{36.49} & \cellcolor{best}\textbf{1.011} & \cellcolor{best}\textbf{0.550}
& \cellcolor{best}\textbf{79.61/79.38} & \cellcolor{best}\textbf{80.62/79.87} & \cellcolor{best}\textbf{48.12} & \cellcolor{best}\textbf{47.06} & \cellcolor{best}\textbf{0.648} & \cellcolor{best}\textbf{0.601} \\
\hline
\end{tabular}
}
\end{small}
\end{center}
\end{table*}

\section{Experiments}
\subsection{Experimental Details}
\paragraph{Datasets.} Extensive experiments are conducted on three standard MSA benchmarks: MOSI~\cite{zadeh2016mosi} and MOSEI~\cite{zadeh2018multimodal} annotated with sentiment scores in $[-3, +3]$. SIMS~\cite{yu2020ch} is a Chinese dataset annotated in $[-1, +1]$. The detailed dataset statistics are introduced in the Appendix.

\paragraph{Implementation Details.}
SemMSA is trained with a 200-epoch schedule using AdamW~\cite{loshchilovdecoupled}, warm-up, and cosine annealing. The batch size (64 or 32) and learning rate ($5\times10^{-5}$ to $2\times10^{-4}$) are set per metric, as listed in the Appendix. The input sequence length $T$ is 8, and the hidden dimension $d$ is 128. Qwen3-1.7B~\cite{yang2025qwen3} is adopted as the frozen LLM to generate semantic features. The visual and acoustic Adapters use $M=8$ learnable prompts and $B=2$ blocks. Refinement step $O$ in CSR is set to 4. The RBF kernel bandwidth $\sigma=1.0$ and temperature $\tau=1.0$ by default. All experiments are performed with NVIDIA RTX 6000 Ada and RTX 4090 GPUs.

\paragraph{Missingness Settings and Evaluation Metrics.} Following~\cite{zhang2024towards, zhu2025proxy,yuan2021transformer, li2025tf}, missing visual and acoustic segments are replaced with zero vectors and language tokens with \texttt{[UNK]}. Training employs instance-wise Bernoulli masking (see Table~\ref{table9}). During testing, the missing rate is varied from $0$ to $0.9$ (step $0.1$) and missing positions are independently sampled per modality. Random seeds 1111, 1112, and 1113 are used, as detailed in the Appendix. For MOSI and MOSEI, we report seven-class accuracy (Acc-7), Acc-5, Acc-2, F1 score, Mean Absolute Error (MAE), and Pearson correlation coefficient (Corr). Acc-2 and F1 are reported sequentially in two forms of negative/non-negative and negative/positive. For SIMS, we report Acc, F1 score, MAE, and Corr. 

\begin{wraptable}{r}{0.5\textwidth}
\vspace{-0.53cm} 
\caption{The average results on SIMS dataset over missing rates from 0.0 to 0.9.}
\label{sota2}
\begin{center}
\begin{small}
\resizebox{\linewidth}{!}{
\begin{tabular}{lcccccc}
\toprule
Method & Acc-2 & F1 & Acc-3 & Acc-5 & MAE $\downarrow$ & Corr $\uparrow$ \\
\midrule
MISA~\cite{hazarika2020misa}
& 72.71 & 66.30 & 56.87 & 31.53 & 0.539 & 0.348 \\

Self-MM~\cite{yu2021learning}
& 72.81 & 68.43 & 56.75 & 32.28 & 0.508 & 0.376 \\

MMIM~\cite{han2021improving}
& 69.86 & 66.21 & 56.76 & 31.81 & 0.544 & 0.339 \\

TETFN~\cite{wang2023tetfn}
& 73.58 & 68.67 & \cellcolor{third}56.91 & 33.42 & \cellcolor{third}0.505 & \cellcolor{third}0.387 \\

TFR-Net~\cite{yuan2021transformer}
& 68.13 & 58.70 & 52.89 & 26.52 & 0.661 & 0.169 \\

ALMT~\cite{zhang2023learning}
& 71.85 & \cellcolor{second}76.21 & 56.47 & 34.16 & 0.509 & 0.372 \\

LNLN~\cite{zhang2024towards}
& 73.15 & 69.18 & \cellcolor{second}57.32 & \cellcolor{third}34.68 & 0.515 & \cellcolor{third}0.387 \\

P-RMF~\cite{zhu2025proxy}
& \cellcolor{third}73.64 & \cellcolor{third}74.65 & 54.75 & \cellcolor{second}34.83 & \cellcolor{second}0.500 & \cellcolor{second}0.414 \\

TF-Mamba~\cite{li2025tf}
& \cellcolor{second}74.68 & 72.20 & 55.51 & 34.46 & 0.512 & 0.386 \\

\hline
\textbf{SemMSA}
& \cellcolor{best}\textbf{75.46}
& \cellcolor{best}\textbf{77.50}
& \cellcolor{best}\textbf{58.84}
& \cellcolor{best}\textbf{35.68}
& \cellcolor{best}\textbf{0.474}
& \cellcolor{best}\textbf{0.501} \\
\hline
\end{tabular}
}
\end{small}
\end{center}
\vspace{-0.5cm}
\end{wraptable}

\subsection{Main Results}

\begin{figure*}[t]
  \centering
  \includegraphics[width=\linewidth]{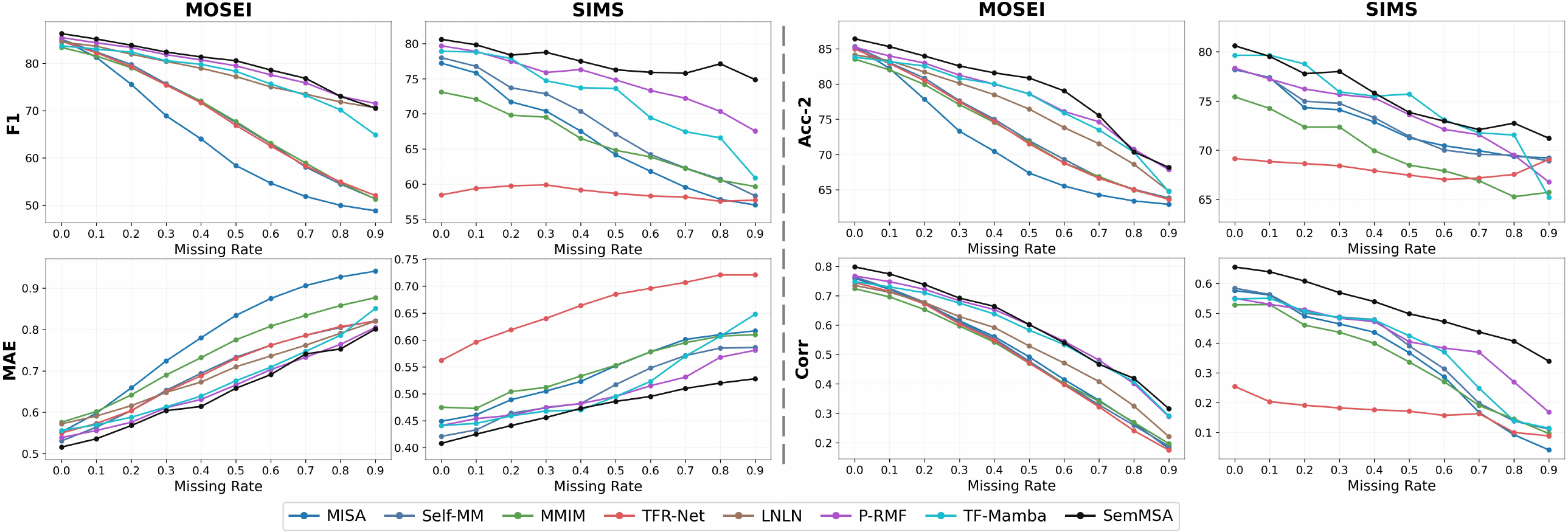} 
  \caption{
Performance visualization of F1 Score, MAE, Acc-2, and Corr across MOSI, MOSEI, and SIMS under missing rates from 0.0 to 0.9, where lower MAE denotes superior performance.}
  \label{fig3}
\end{figure*}

\paragraph{Intra-modal Missingness.} 
Table~\ref{sota1} and~\ref{sota2} report comparisons on MOSI, MOSEI, and SIMS under intra-modal missing settings. SemMSA consistently achieves state-of-the-art performance. On MOSI, it improves Acc-5 and Acc-7 over the best baselines TF-Mamba~\cite{li2025tf} by 2.43\% and 2.23\%, respectively, while achieving the lowest MAE of 1.011. On MOSEI, SemMSA further improves Acc-2 to 79.61 and F1 to 80.62, demonstrating its scalability on larger-scale data. On SIMS, the improvement is particularly clear, surpassing P-RMF~\cite{zhu2025proxy} by 0.087 in Corr. Figure~\ref{fig3} further illustrates the performance trends of F1, Acc-2, MAE, and Corr.  As the missing rate increases, most baseline methods exhibit clear performance degradation, struggling to maintain reliable representations.  In contrast, SemMSA maintains consistently superior and more stable performance across missing rates. These results indicate that SemMSA can effectively complement incomplete multimodal observations with sentiment-relevant latent semantics and comprehensively align heterogeneous representations through global nonlinear alignment, leading to strong robustness and generalization under challenging missing-modality scenarios. Detailed results under each missing rate are provided in the Appendix.

\begin{wraptable}{r}{0.5\textwidth}
\vspace{-0.4cm} 
\caption{The F1 scores on the MOSEI dataset under varying inter-modal missing conditions.}
\label{sota3}
\begin{center}
\begin{small}
\setlength{\tabcolsep}{3pt}
\renewcommand{\arraystretch}{1.05}
\resizebox{\linewidth}{!}{
\begin{tabular}{c cccccccc}
\toprule
Method 
& $\{l\}$ & $\{a\}$ & $\{v\}$ & $\{l,a\}$ & $\{l,v\}$ & $\{a,v\}$ & $\{l,a,v\}$ & Avg. \\
\midrule

Self-MM~\cite{yu2021learning}
& 71.53 & 43.57 & 37.61 & 75.91 & 74.62 & 49.52 & \cellcolor{third}83.69 & 62.35 \\

CubeMLP~\cite{sun2022cubemlp}
& 67.52 & 39.54 & 32.58 & 71.69 & 70.06 & 48.54 & 83.17 & 59.01 \\

DMD~\cite{li2023decoupled}
& 70.26 & 46.18 & 39.84 & 74.78 & 72.45 & 52.70 & \cellcolor{second}84.78 & 63.00 \\

MCTN~\cite{pham2019found}
& 75.50 & 62.72 & 59.46 & 76.64 & 77.13 & 64.84 & 81.75 & 71.15 \\

TransM~\cite{wang2020transmodality}
& \cellcolor{third}77.98 & 63.68 & 58.67 & \cellcolor{third}80.46 & 78.61 & 62.24 & 81.48 & 71.87 \\

SMIL~\cite{ma2021smil}
& 76.57 & \cellcolor{third}65.96 & 60.57 & 77.68 & 76.24 & 66.87 & 80.74 & 72.09 \\

GCNet~\cite{lian2023gcnet}
& \cellcolor{second}80.52 & \cellcolor{second}66.54 & \cellcolor{second}61.83 & \cellcolor{second}81.96 & \cellcolor{second}81.15 & \cellcolor{third}69.21 & 82.35 & \cellcolor{third}74.79 \\

CorrKD~\cite{li2024correlation}
& 80.76 & 66.09 & \cellcolor{third}62.30 & 81.74 & \cellcolor{third}81.28 & \cellcolor{second}71.92 & 82.16 & \cellcolor{second}75.18 \\

\hline
\textbf{SemMSA} 
& \cellcolor{best}\textbf{83.61} 
& \cellcolor{best}\textbf{66.59} 
& \cellcolor{best}\textbf{65.91} 
& \cellcolor{best}\textbf{85.17} 
& \cellcolor{best}\textbf{84.96} 
& \cellcolor{best}\textbf{72.68} 
& \cellcolor{best}\textbf{86.32} 
& \cellcolor{best}\textbf{77.89} \\
\hline

\end{tabular}
}
\end{small}
\end{center}
\vspace{-0.4cm}
\end{wraptable}
\paragraph{Inter-modal Missingness.}

Following the same protocol in~\cite{li2024correlation, lian2023gcnet, li2023decoupled}, we further remove entire modalities from MOSEI samples to evaluate the performance of SemMSA under inter-modal missingness. The results are reported in Table~\ref{sota3} using F1 as the evaluation metric. It is worth noting that (1) all methods degrade significantly in unimodal scenarios, confirming the complementary roles of different modalities. (2) Language consistently outperforms visual and acoustic modalities, highlighting more explicit sentiment semantics in textual information. (3) SemMSA achieves the best results under all inter-modal missingness, outperforming the second-best CorrKD~\cite{li2024correlation} by a significant margin of 2.71\% on average.  Improvements in the language missingness, such as \{a,v\}, demonstrate that SemMSA effectively leverages nonverbal evidence to obtain rich sentiment-relevant semantics by semantic refinement, fully integrating them via anchor-free  spectral alignment.

\begin{table*}[htbp]
\caption{Ablation study of components and losses on MOSI and SIMS datasets. }
\label{ablation}
\begin{center}
\begin{small}
\setlength{\tabcolsep}{3pt}
\renewcommand{\arraystretch}{1.05}
\resizebox{\textwidth}{!}{
\begin{tabular}{ccc cccccc cccccc}
\toprule

\multirow{2}{*}{\textbf{CSR}} 
& \multirow{2}{*}{$\boldsymbol{\mathcal{L}_{\mathrm{csa}}}$} 
& \multirow{2}{*}{$\boldsymbol{\mathcal{L}_{\mathrm{sep}}}$} 
& \multicolumn{6}{c}{MOSI} 
& \multicolumn{6}{c}{SIMS} \\

\cmidrule(lr){4-9} \cmidrule(lr){10-15}

&&
& Acc-2 & F1 & Acc-5 & Acc-7 & MAE $\downarrow$ & Corr $\uparrow$
& Acc-2 & F1 & Acc-3 & Acc-5 & MAE $\downarrow$ & Corr $\uparrow$ \\

\midrule

% none
& & 
& 67.49/69.47 & 63.28/67.01 & 34.29 & 31.71 & 1.085 & 0.501
& 66.83 & 57.92 & 31.12 & 54.61 & 0.593 & 0.387 \\

% CSR
\checkmark & & 
& 73.36/73.29 & 73.04/73.28 & 38.84 & 35.03 & 1.040 & 0.522
& 74.18 & 76.21 & 34.27 & 57.36 & 0.488 & 0.450 \\

% CSR + Lcsa
\checkmark & \checkmark & 
& 73.94/73.56 & 73.83/73.46 & 40.48 & 36.08 & 1.024 & 0.529
& 74.92 & 76.88 & 35.02 & 58.12 & 0.480 & 0.493 \\

% full
\checkmark & \checkmark & \checkmark
& \textbf{74.36/73.91} & \textbf{74.18/73.82} & \textbf{40.93} & \textbf{36.49} & \textbf{1.011} & \textbf{0.550}
& \textbf{75.46} & \textbf{77.50} & \textbf{35.68} & \textbf{58.84} & \textbf{0.474} & \textbf{0.501} \\

\bottomrule
\end{tabular}
}
\end{small}
\end{center}
\end{table*}

\subsection{Model Analysis}

\paragraph{Ablation Study.} We conduct an ablation study on MOSI and SIMS to evaluate the effectiveness of the proposed CSR, $\mathcal{L}_{csa}$, and $\mathcal{L}_{sep}$, as shown in Table~\ref{ablation}. First, introducing CSR alone brings substantial gains over the baseline on both datasets, improving Acc-2 and F1 by 5.68\% and 11.11\% on average. These indicate that rich sentiment-relevant semantics effectively compensate for incomplete inputs. Second, adding $\mathcal{L}_{\mathrm{csa}}$ further brings performance gains. MOSI Acc-5 increases from 38.84 to 40.48 and Acc-7 from 35.03 to 36.08, suggesting better cross-modal consistency. Finally, the best performance is achieved when $\mathcal{L}_{sep}$ is also incorporated. The full model consistently outperforms all ablated variants, highlighting the importance of semantic refinement for enhancing discriminative sentiment semantics, as well as the complementarity between intra-instance cross-modal coherence by $\mathcal{L}{\mathrm{csa}}$ and inter-instance discriminability by $\mathcal{L}{\mathrm{sep}}$.

\begin{table}[t]
\centering
\caption{Comparison with different projection and alignment methods on SIMS dataset.}
\label{ab_rl}

\begin{small}
\setlength{\tabcolsep}{2.5pt}
\renewcommand{\arraystretch}{1.05}

\begin{subtable}{0.49\textwidth}
\centering
\caption{Projection mechanisms}
\resizebox{\textwidth}{!}{
\begin{tabular}{lcccccc}
\toprule
Method & Acc-2 & F1 & Acc-5 & Acc-3 & MAE $\downarrow$ & Corr $\uparrow$ \\
\midrule
Linear       & 72.84 & 73.62 & 54.91 & 32.46 & 0.512 & 0.431 \\
MLP          & 73.58 & 74.36 & 56.03 & 33.42 & 0.501 & 0.452 \\
Trans.~\cite{vaswani2017attention} & 74.52 & 76.21 & 57.63 & 34.76 & 0.486 & 0.482 \\
QFormer~\cite{li2023blip} & 74.83 & 76.68 & 58.12 & 35.09 & 0.481 & 0.491 \\
Ours         & \textbf{75.46} & \textbf{77.50} & \textbf{58.84} & \textbf{35.68} & \textbf{0.474} & \textbf{0.501} \\
\bottomrule
\end{tabular}
}
\label{ab1}
\end{subtable}
\hfill
\begin{subtable}{0.49\textwidth}
\centering
\caption{Alignment mechanisms}
\resizebox{\textwidth}{!}{
\begin{tabular}{lcccccc}
\toprule
Method & Acc-2 & F1 & Acc-5 & Acc-3 & MAE $\downarrow$ & Corr $\uparrow$ \\
\midrule
InfoNCE~\cite{oord2018representation}
& 73.42 & 75.39 & 56.88 & 33.74 & 0.501 & 0.489 \\
CMD~\cite{hazarika2020misa}
& 73.77 & 75.18 & 56.92 & 34.21 & 0.503 & 0.487 \\
PMRL~\cite{liu2026principled}
& 74.63 & 75.98 & 57.62 & 34.58 & 0.493 & 0.499 \\
Volume~\cite{cicchettigramian}
& 74.31 & 76.35 & 57.91 & 34.86 & 0.487 & 0.506 \\
Ours
& \textbf{75.46} & \textbf{77.50} & \textbf{58.84} & \textbf{35.68} & \textbf{0.474} & \textbf{0.501} \\
\bottomrule
\end{tabular}
}
\label{ab2}
\end{subtable}

\end{small}
\end{table}

\paragraph{Comparison with Projection Methods.}
As shown in Table~\ref{ab1}, we compare different projection strategies by replacing our Adapter with representative alternatives. Directly projecting modality features into the LLM embedding space with linear or MLP leads to a clear performance drop, indicating that dimensional matching alone is insufficient for constructing effective multimodal prefixes. Despite the effectiveness of Standard Transformer ($\sim$85M)~\cite{vaswani2017attention} and Q-Former ($\sim$100M)~\cite{li2023blip}, they introduce substantially higher computational overhead. In contrast, our lightweight Adapter with only 4.7M parameters achieves the best overall results, indicating that it can compactly aggregate sequence-level modal evidence and adaptively select reliable clues under incomplete observations.

\paragraph{Comparison with Alignment Methods.} As shown in Table ~\ref{ab2}, we compare four cross-modal alignment strategies. The baseline is utilized without an alignment objective. InfoNCE~\cite{oord2018representation} and CMD~\cite{hazarika2020misa} rely on pairwise similarity, failing to model the joint structure across all three modalities, resulting in poor results. While the PMRL~\cite{liu2026principled} and volume-based loss improve upon this, they remain limited to the linear space and anchor modality, struggling to capture complex semantic relationships and select a stable anchor under severe intra-modality missingness. In contrast, our method introduces a Reproducing Kernel Hilbert Space (RKHS) without anchor dependency to enable nonlinear and robust alignment, achieving the best results.

\begin{table}[htbp]
\centering
\caption{Efficiency comparison on MOSI.}
\label{table6}
\begin{small}
\setlength{\tabcolsep}{4pt}
\renewcommand{\arraystretch}{1.1}
\begin{tabular}{lcccccc}
\toprule
Method & Trainable Params & Task GFLOPs & Memory(GB) & Latency(ms) & Acc-2(\%) & F1(\%) \\
\midrule
LNLN~\cite{zhang2024towards} & 116M & 9.0 & 12.8 & 25.1 & 70.94/72.55 & 71.25/72.73 \\
P-RMF~\cite{zhu2025proxy} & 117M & 9.7 & 13.5 & 67.0 & 71.53/72.81 & 71.69/72.93  \\
SemMSA & 113M & 4.8 & 13.0 & 30.6 & 74.36/73.91 & 74.18/73.82 \\
\bottomrule
\end{tabular}
\end{small}
\end{table}

\paragraph{Comparison of Computational Overhead.}
Efficiency experiments are further conducted on MOSI using the same server configuration as LNLN~\cite{zhang2024towards} and P-RMF~\cite{zhu2025proxy}, where trainable parameters and GFLOPs include only the modules optimized for MSA. As shown in Table~\ref{table6}, SemMSA achieves the best performance while maintaining a comparable computation cost. It reduces the task GFLOPs by 50.5\% and inference time by 54.3\% compared with P-RMF. Notably, the reported memory and latency reflect the entire pipeline with the frozen LLM, rather than just the trainable modules. The efficiency mainly comes from the token-efficient hidden-state refinement design. SemMSA avoids autoregressive decoding of explicit textual descriptions and appends only a small number of continuous latent states. These results indicate that SemMSA provides a favorable trade-off between robustness and inference cost under incomplete multimodal settings.

\begin{figure*}[t]
  \centering
  \includegraphics[width=\linewidth]{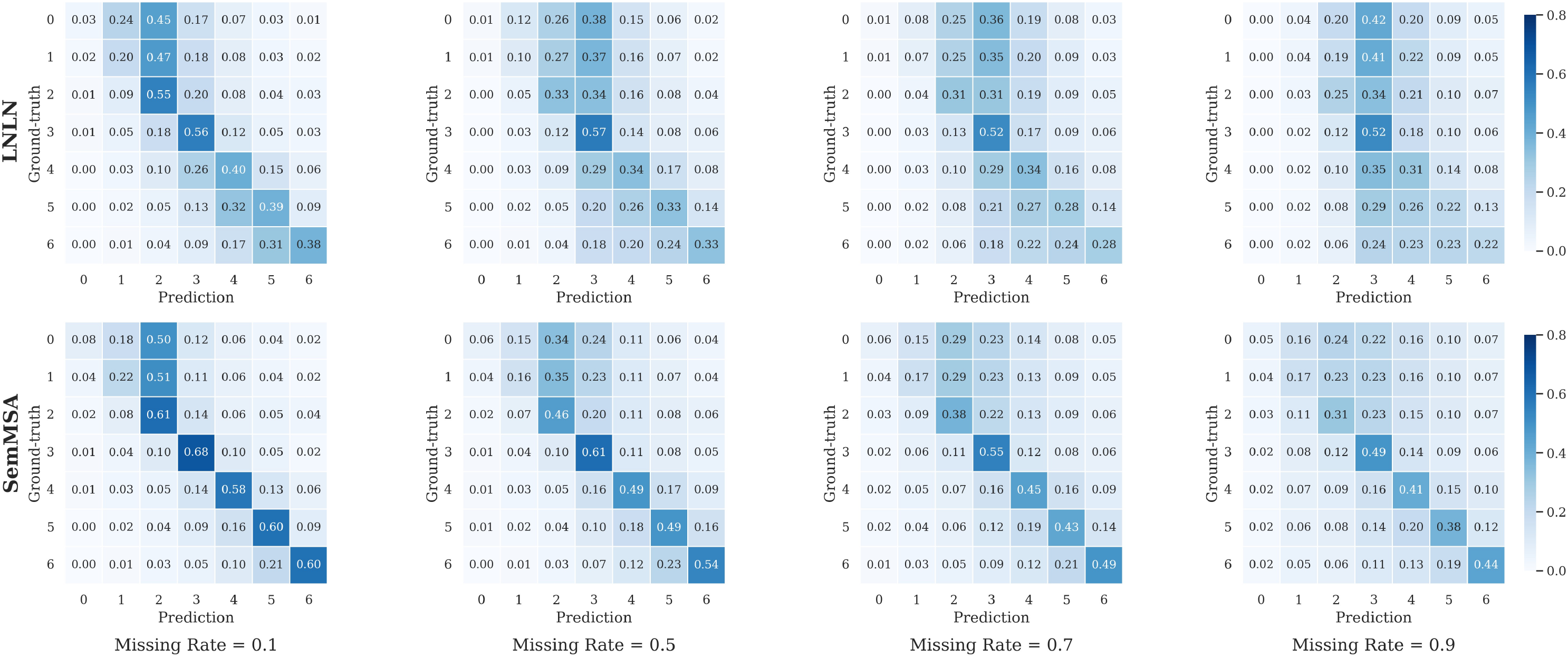} 
  \caption{ Confusion matrices of SemMSA and LNLN on the MOSI, where 0-6 denote strongly negative, negative, weakly negative, neutral, weakly positive, positive, and strongly positive, respectively.}
  \label{fig4}
  \vspace{-0.15cm}
\end{figure*}

\paragraph{Visualization of Prediction Performance.} To further evaluate the robustness of SemMSA, we visualize the confusion matrices on the MOSI dataset under missing rates of 0.1, 0.5, 0.7, and 0.9, as shown in Figure~\ref{fig4}.  Under high missing rates, LNLN tends to wrongly concentrate predictions around middle sentiment classes. In contrast, SemMSA maintains clearer diagonal patterns across sentiment categories. This demonstrates that the semantic-aware modeling in SemMSA facilitates the capture of discriminative sentiment cues from incomplete inputs, alleviating classification boundary degeneration for robust prediction under severe modality missingness. 

\vspace{-0.2cm}

\begin{table}[htbp]
\centering
\caption{Comparison with LLMs for prediction and semantic generation on MOSI and SIMS.}
\label{table7}
\begin{small}
\setlength{\tabcolsep}{4pt}
\renewcommand{\arraystretch}{1.05}
\resizebox{\columnwidth}{!}{
\begin{tabular}{llcccccc cccccc}
\toprule
\multirow{2}{*}{Model}
& \multirow{2}{*}{Role}
& \multicolumn{6}{c}{MOSI}
& \multicolumn{6}{c}{SIMS} \\
\cmidrule(lr){3-8} \cmidrule(lr){9-14}
& & Acc-2 & F1 & Acc-5 & Acc-7 & MAE $\downarrow$ & Corr $\uparrow$
& Acc-2 & F1 & Acc-5 & Acc-3 & MAE $\downarrow$ & Corr $\uparrow$ \\
\midrule

Qwen2.5-Omni-7B~\cite{xu2025qwen25omnitechnicalreport}
& Prediction
& 69.84/71.36 & 69.92/71.58 & 36.94 & 33.21 & 1.086 & 0.497
& 70.83 & 70.26 & 52.74 & 31.92 & 0.548 & 0.341 \\
\midrule

Qwen3-1.7B~\cite{yang2025qwen3}
& Semantics
& 74.36/73.91 & 74.18/\textbf{73.82} & 40.93 & 36.49 & 1.011 & 0.550
& 75.46 & 77.50 & 58.84 & 35.68 & 0.474 & 0.501 \\

Llama3.1-8B~\cite{grattafiori2024llama}
& Semantics
& 74.42/73.75 & 74.04/73.69 & \textbf{41.11} & 36.50 & 1.015 & 0.546
& - & - & - & - & - & - \\

Qwen3-8B~\cite{yang2025qwen3}
& Semantics
& \textbf{74.66}/73.88 & 74.11/73.79 & 40.86 & \textbf{36.62} & 1.012 & \textbf{0.553}
& 75.38 & \textbf{77.63} & 58.79 & \textbf{35.77} & 0.476 & \textbf{0.514} \\

Qwen2.5-Omni-7B~\cite{xu2025qwen25omnitechnicalreport}
& Semantics
& 74.52/\textbf{74.03} & \textbf{74.26}/73.76 & 41.05 & 36.55 & \textbf{1.008} & 0.548
& \textbf{75.59} & 77.42 & \textbf{58.96} & 35.61 & \textbf{0.471} & 0.512 \\

\bottomrule
\end{tabular}
}
\end{small}
\vspace{-0.2cm}
\end{table}

\paragraph{The Effect of Different LLMs.} Table~\ref{table7} compares open-source Qwen3-1.7B, Llama-3.1-8B, Qwen3-8B, and Qwen2.5-Omni-7B on MOSI and SIMS. For fair comparison, the direct prediction of LLMs strictly follows the same evaluation protocol and data splits for MSA with incomplete data. However, its performance is relatively limited with only 69.84/71.36 Acc-2 and 69.92/71.58 F1 on MOSI, showing the difficulty of adapting LLMs to sentiment prediction.  In contrast, using LLMs for semantic generation brings clear improvements, since rich semantic information provides auxiliary high-level knowledge to compensate for scarce representations with missing data. Specifically, Qwen3-8B achieves the best F1 and Acc-3 on SIMS. Llama3.1-8B without support of Chinese and Qwen2.5-Omni-7B obtains superior results of Acc-5 and MAE on MOSI, respectively. These results show that SemMSA is robust across different LLMs and that LLMs are more effective as semantic generators than direct predictors for MSA with incomplete data.
\section{Conclusion}
In this paper, we propose SemMSA, a semantic-aided framework to enrich incomplete data with LLM-derived sentiment-aware semantic representations and integrate all modalities through anchor-free spectral alignment. Specifically, we introduce Cross-modal Semantic Refinement (CSR), which maps visual and acoustic representations into the frozen LLM embedding space with lightweight Adapters. CSR then produces sentiment-relevant latent semantics iteratively through token-efficient hidden-state refinement, without decoding explicit text. We further propose Cross-modal Spectral Alignment (CSA), which jointly aligns semantic, language, visual, and acoustic representations by enhancing the dominant spectral component of their kernel Gram matrix. This captures global nonlinear cross-modal dependencies without relying on a predefined anchor modality, while preserving discriminative structure through instance-level spectral separation. Extensive experiments on SIMS, MOSI, and MOSEI datasets demonstrate the effectiveness of SemMSA, achieving state-of-the-art performance under diverse missing-modality settings.

\section*{Acknowledgment}
This work was supported in part by National Natural Science Foundation of China under Grant 62406177, in part by the Shandong Excellent Young Scientists Fund (Oversea) under Grant 2024HWYQ-027, in part by the Natural Science Foundation of Shandong province, China under Grant ZR2023QF124, in part by the Young Scholars Program of Shandong University.

{
\small
\bibliographystyle{IEEEtran}
\bibliography{reference}
}

\newpage
\appendix
\section{Appendix / supplemental material}

\subsection{Training Algorithm}

\begin{algorithm2e}
	\caption{Training algorithm of the proposed SemMSA.}
	\label{alg}
	\KwIn{Training set $\mathcal{D}_{train}$, feature encoders $f_{\phi}^{m}$, frozen LLM $F_{\theta}$, visual/audio Adapters $\mathcal{A}_{v},\mathcal{A}_{a}$, modality encoders $\mathcal{E}_{m}$, semantic projector $g_s$, regressor $\mathcal{C}$}
	\While{not converged}
	{
		$1.$ Sample a batch $\mathcal{B}=\{(U_v^i,U_a^i,U_l^i,y^i)\}_{i=1}^{N}$ from $\mathcal{D}_{train}$ and apply instance-wise missing-modality masking; \\

		$2.$ Extract modality features $X_m^i$ according to Eq.~\ref{eq1}; \\

		$3.$ Feed $X_v^i$ and $X_a^i$ into visual/audio Adapters to obtain compact prefix tokens $Z_v^i$ and $Z_a^i$, and embed language input as $Z_l^i$; \\

		$4.$ Construct the initial multimodal prefix $U^{i,(0)}=[Z_l^i;Z_v^i;Z_a^i]$ in the frozen LLM embedding space; \\

		\For{$k=1$ \KwTo $O$}
		{
			$5.$ Obtain the latent semantic state $z_k^i$ from the final-position hidden state of $F_{\theta}$, and update $U^{i,(k)}$ according to Eq.~\ref{eq2} and Eq.~\ref{eq3}; \\
		}

		$6.$ Collect $Z^i$ and obtain the fused semantic representation according to Eq.~\ref{eq4} and Eq.~\ref{eq5}; \\

		$7.$ Extract visual, acoustic, and language features to obtain $H_v^i$, $H_a^i$, and $H_l^i$; \\

		$8.$ Normalize $H_s^i,H_v^i,H_a^i,H_l^i$ and construct $V^i=[h_s^i,h_v^i,h_a^i,h_l^i]$; \\

		$9.$ Build the kernel Gram matrix $K^i$ and perform eigendecomposition according to Eq.~\ref{eq6}--Eq.~\ref{eq8}; \\

		$10.$ Compute the cross-modal spectral alignment loss $\mathcal{L}_{\text{csa}}$ by enhancing the dominance of $\lambda_1^i$ according to Eq.~\ref{eq9}; \\

		$11.$ Compute the spectral separation loss $\mathcal{L}_{\text{sep}}$ from dominant semantic directions $\{u_1^i\}_{i=1}^{N}$ according to Eq.~\ref{eq10}; \\

		$12.$ Fuse all representations as $\hat{H}^i$ and predict sentiment score $\hat{y}^i$ according to Eq.~\ref{eq11}; \\

		$13.$ Calculate $\mathcal{L}_{\text{total}}$ with the sentiment regression loss, $\mathcal{L}_{\text{csa}}$, and $\mathcal{L}_{\text{sep}}$ according to Eq.~\ref{eq12}; \\

		$14.$ Update all trainable parameters via gradient backpropagation. \\
	}
	\KwOut{Trained SemMSA model for MSA with incomplete data.}
\end{algorithm2e}

We describe the training procedure of SemMSA in Algorithm~\ref{alg}. 
Given incomplete multimodal inputs, SemMSA first extracts modality-specific features and maps visual and acoustic evidence into the frozen LLM embedding space through lightweight Adapters. 
Then, Cross-modal Semantic Refinement (CSR) iteratively appends hidden semantic states from the frozen LLM to produce compact sentiment-aware semantics without explicit text decoding. 
After that, Cross-modal Spectral Alignment (CSA) constructs a kernel Gram matrix over semantic, visual, acoustic, and language representations, and enhances its dominant spectral component for anchor-free multimodal alignment. 
An instance-level spectral separation loss is further introduced to preserve discriminability across samples. 
Finally, all representations are fused, and the model is optimized for sentiment prediction.

\subsection{Theoretical Analysis}
\label{app:csa_theory}

In this section, we provide a theoretical analysis of the proposed Cross-modal Spectral Alignment (CSA).
Different from pairwise contrastive alignment, CSA aligns semantic, visual, acoustic, and language representations simultaneously by analyzing the spectral structure of their kernel Gram matrix.
The analysis shows that enhancing the dominant spectral component of the kernel Gram matrix encourages all modalities to concentrate around a shared nonlinear semantic direction in the RKHS, thereby achieving anchor-free cross-modal alignment.

\subsubsection{Preliminaries}

For the $i$-th multimodal instance, SemMSA obtains four normalized representations:
\begin{equation}
    V^i = [h_s^i, h_v^i, h_a^i, h_l^i],
\end{equation}
where $h_s^i$, $h_v^i$, $h_a^i$, and $h_l^i$ denote the semantic, visual, acoustic, and language representations, respectively.
For simplicity, we omit the instance index $i$ and denote the $k$ modality representations as $\{h_1,h_2,\ldots,h_k\}$
CSA maps these representations into an RKHS $\mathcal{F}$ through a nonlinear feature map $\psi(\cdot)$ induced by a positive definite kernel $\kappa(\cdot,\cdot)$:
\begin{equation}
    \kappa(h_p,h_q)=\langle \psi(h_p),\psi(h_q)\rangle_{\mathcal{F}}.
\end{equation}
In SemMSA, we adopt the RBF kernel:
\begin{equation}
    \kappa(h_p,h_q)=\exp\left(-\frac{\|h_p-h_q\|_2^2}{2\sigma^2}\right).
\end{equation}
The corresponding kernel Gram matrix is defined as
\begin{equation}
    K =
    \begin{bmatrix}
        \kappa(h_1,h_1) & \cdots & \kappa(h_1,h_k) \\
        \vdots & \ddots & \vdots \\
        \kappa(h_k,h_1) & \cdots & \kappa(h_k,h_k)
    \end{bmatrix}
    \in \mathbb{R}^{k\times k}.
\end{equation}
Since the RBF kernel satisfies $\kappa(h_p,h_p)=1$, we have
\begin{equation}
    \mathrm{Tr}(K)=k.
\end{equation}
As $K$ is symmetric positive semi-definite, it admits the eigendecomposition:
\begin{equation}
    K = Q\Lambda Q^\top,\quad
    \Lambda=\mathrm{diag}(\lambda_1,\lambda_2,\ldots,\lambda_k),
\end{equation}
where $\lambda_1\geq\lambda_2\geq\cdots\geq\lambda_k\geq0$ and $\sum_{j=1}^{k}\lambda_j=k$.

\subsubsection{Kernelized Full Alignment and Rank-one Gram Matrix}

\begin{definition}[Kernelized full alignment]
The modality representations $\{h_j\}_{j=1}^{k}$ are fully aligned in the RKHS if their kernel embeddings are identical:
\begin{equation}
    \psi(h_1)=\psi(h_2)=\cdots=\psi(h_k).
\end{equation}
\end{definition}

\begin{lemma}[Kernelized full alignment $\Longleftrightarrow$ rank-one kernel Gram matrix]
\label{lemma:kernel_rank_one}
Let $K\in\mathbb{R}^{k\times k}$ be the kernel Gram matrix constructed from $\{\psi(h_j)\}_{j=1}^{k}$ with $\kappa(h_j,h_j)=1$.
Then the following two statements are equivalent:
\begin{equation}
    \psi(h_1)=\psi(h_2)=\cdots=\psi(h_k),
\end{equation}
and
\begin{equation}
    \mathrm{rank}(K)=1, \quad K=\mathbf{1}\mathbf{1}^{\top}.
\end{equation}
\end{lemma}

\begin{proof}
If all modality embeddings are fully aligned in the RKHS, then for any $p,q\in\{1,\ldots,k\}$,
\begin{equation}
    K_{pq}=\langle \psi(h_p),\psi(h_q)\rangle_{\mathcal{F}}
    =\langle \psi(h_p),\psi(h_p)\rangle_{\mathcal{F}}=1,
\end{equation}
where the last equality follows from $\kappa(h_p,h_p)=1$.
Therefore, $K=\mathbf{1}\mathbf{1}^{\top}$, which is a rank-one matrix.

Conversely, suppose $\mathrm{rank}(K)=1$ and $\kappa(h_j,h_j)=1$ for all $j$.
Since $K$ is symmetric positive semi-definite, there exists a vector $r\in\mathbb{R}^{k}$ such that
\begin{equation}
    K=rr^\top.
\end{equation}
The diagonal constraint $K_{jj}=1$ implies $r_j^2=1$ for all $j$.
For the RBF kernel, all entries satisfy $K_{pq}>0$.
Thus $r_p r_q>0$ for any $p,q$, meaning all elements of $r$ share the same sign.
Consequently,
\begin{equation}
    K_{pq}=r_p r_q=1,\quad \forall p,q.
\end{equation}
Hence $K=\mathbf{1}\mathbf{1}^{\top}$.
Since
\begin{equation}
    \|\psi(h_p)-\psi(h_q)\|_{\mathcal{F}}^2
    =K_{pp}+K_{qq}-2K_{pq}
    =1+1-2=0,
\end{equation}
we obtain $\psi(h_p)=\psi(h_q)$ for all modality pairs.
Therefore, all modalities are fully aligned in the RKHS.
\end{proof}

Lemma~\ref{lemma:kernel_rank_one} shows that the goal of nonlinear cross-modal alignment can be transformed into encouraging the kernel Gram matrix to approach a rank-one structure.
This provides the theoretical basis for CSA.

\subsubsection{Dominant Spectral Component Encourages Nonlinear Alignment}

\begin{theorem}[Dominant eigenvalue maximization promotes kernelized cross-modal alignment]
\label{thm:dominant_eigen}
Let $K$ be the kernel Gram matrix of $k$ normalized modality representations in the RKHS, with eigenvalues
$\lambda_1\geq\lambda_2\geq\cdots\geq\lambda_k\geq0$ and $\sum_{j=1}^{k}\lambda_j=k$.
Then $K$ is rank-one if and only if
\begin{equation}
    \lambda_1=k,\quad \lambda_2=\cdots=\lambda_k=0.
\end{equation}
Moreover, increasing the dominance ratio of $\lambda_1$ reduces the residual spectral energy
\begin{equation}
    \sum_{j=2}^{k}\lambda_j = k-\lambda_1,
\end{equation}
thereby pushing $K$ toward its optimal rank-one approximation.
\end{theorem}

\begin{proof}
Since $K$ is positive semi-definite, all eigenvalues are non-negative.
The rank of $K$ equals the number of non-zero eigenvalues.
Thus, $K$ is rank-one if and only if exactly one eigenvalue is non-zero.
Because $\mathrm{Tr}(K)=\sum_{j=1}^{k}\lambda_j=k$, the only possible rank-one spectrum is
\begin{equation}
    \lambda_1=k,\quad \lambda_2=\cdots=\lambda_k=0.
\end{equation}

For the second statement, the residual spectral energy outside the dominant component is
\begin{equation}
    \sum_{j=2}^{k}\lambda_j
    =
    \sum_{j=1}^{k}\lambda_j-\lambda_1
    =
    k-\lambda_1.
\end{equation}
Therefore, increasing $\lambda_1$ directly decreases the total energy of all non-dominant spectral components.
When $\lambda_1$ approaches $k$, the remaining eigenvalues approach zero, and $K$ approaches a rank-one matrix.
According to Lemma~\ref{lemma:kernel_rank_one}, this corresponds to full cross-modal alignment in the RKHS.
\end{proof}

Theorem~\ref{thm:dominant_eigen} justifies the CSA objective in SemMSA.
Specifically, CSA treats the eigenvalues of $K$ as logits and enhances the dominance of the largest eigenvalue:
\begin{equation}
    \mathcal{L}_{\mathrm{csa}}
    =
    -\frac{1}{N}\sum_{i=1}^{N}
    \log
    \frac{\exp(\lambda_1^i/\tau)}
    {\sum_{j=1}^{k}\exp(\lambda_j^i/\tau)}.
\end{equation}
Minimizing $\mathcal{L}_{\mathrm{csa}}$ increases the relative dominance of $\lambda_1^i$ over the remaining eigenvalues.
As a result, the kernel Gram matrix becomes closer to a rank-one matrix, encouraging semantic, visual, acoustic, and language representations to align along a shared nonlinear semantic direction.

\subsubsection{Anchor-free Property}

Pairwise contrastive learning usually selects one modality as the anchor and aligns the remaining modalities to it.
In contrast, CSA does not require a predefined anchor modality.
The dominant eigenvector $q_1$ of $K$ is determined by the joint spectral structure of all modalities:
\begin{equation}
    Kq_1=\lambda_1 q_1.
\end{equation}
Thus, the shared alignment direction is induced by all modality representations rather than by any single modality.
The corresponding dominant semantic direction in the representation space is computed as
\begin{equation}
    u_1=\frac{Vq_1}{\|Vq_1\|_2}.
\end{equation}
This direction adaptively summarizes the principal cross-modal consensus of each instance.
Therefore, CSA performs any-to-any alignment by enhancing a data-dependent dominant spectral component, avoiding the instability caused by manually selecting an anchor modality.

\subsubsection{Instance-level Spectral Separation}

Although maximizing the dominant eigenvalue promotes intra-instance cross-modal alignment, it may lead to a degenerate solution where different instances collapse to a similar semantic direction.
To avoid this issue, SemMSA introduces an instance-level spectral separation loss:
\begin{equation}
    \mathcal{L}_{\mathrm{sep}}
    =
    \frac{1}{N(N-1)}
    \sum_{i=1}^{N}
    \sum_{\substack{j=1 \\ j\neq i}}^{N}
    \left[
    (u_1^i)^\top u_1^j
    \right]^2.
\end{equation}
This loss penalizes high similarity between dominant semantic directions from different instances.
Therefore, CSA simultaneously encourages:
\begin{equation}
    \text{intra-instance alignment: }
    \lambda_1^i \gg \lambda_2^i,\ldots,\lambda_k^i,
\end{equation}
and
\begin{equation}
    \text{inter-instance separation: }
    (u_1^i)^\top u_1^j \rightarrow 0,\quad i\neq j.
\end{equation}
The former ensures that different modalities of the same sample are aligned in the RKHS, while the latter preserves discriminability across samples and mitigates representation collapse.

The above analysis indicates that CSA has three desirable theoretical properties.
First, by constructing a kernel Gram matrix, CSA extends linear Gram-based alignment to nonlinear RKHS alignment, enabling the model to capture complex cross-modal dependencies.
Second, by enhancing the dominant eigenvalue, CSA explicitly encourages the kernel Gram matrix to approach a rank-one structure, which corresponds to full alignment of semantic, visual, acoustic, and language representations in the RKHS.
Third, by using the dominant eigenvector rather than a predefined modality as the alignment direction, CSA realizes anchor-free multimodal alignment.
Together with the instance-level spectral separation loss, CSA learns representations that are both cross-modally coherent and instance-discriminative.

\subsection{Experiments} 

According to the granularity of missing elements, MSA with incomplete data can be categorized into inter-modal and intra-modal missingness. The former removes entire modalities, whereas the latter corrupts partial language tokens, visual frames, or acoustic segments within each modality. Intra-modal missingness often occurs randomly and may co-exist across multiple modalities, causing fragmented emotional cues and cross-modal inconsistency. Consequently, intra-modal missingness is considered more complex, mainly focused in this work.

\begin{table}[t]
    \centering
    \caption{Statistics of the multimodal sentiment analysis datasets.}
    \label{table8}
    \scalebox{1}{
        \begin{tabular}{lcccccc}
            \toprule
            \multirow{2}{*}{\textbf{Dataset}} 
            & \multirow{2}{*}{\textbf{Speaker}} 
            & \multirow{2}{*}{\textbf{Clip}} 
            & \multicolumn{3}{c}{\textbf{Sample}} 
            & \multirow{2}{*}{\textbf{Language}} \\
            \cmidrule(lr){4-6}
            & & 
            & $\text{Train}$ 
            & $\text{Valid}$ 
            & $\text{Test}$ 
            & \\
            \cmidrule(lr){1-1} \cmidrule(lr){2-2} \cmidrule(lr){3-3} \cmidrule(lr){4-6} \cmidrule(lr){7-7}
            MOSI~\cite{zadeh2016mosi}
            & 93 & 2,199 & 1,284 & 229 & 686 & English \\
            
            MOSEI~\cite{zadeh2018multimodal}
            & 1,000 & 22,856 & 16,326 & 1,871 & 4,659 & English \\
            
            SIMS~\cite{yu2020ch}
            & 474 & 2,281 & 1,368 & 456 & 457 & Chinese \\
            \bottomrule
        \end{tabular}
    }
\end{table}

\subsubsection{Dataset Details}
SemMSA is evaluated on three standard benchmarks of \textbf{MOSI}~\cite{zadeh2016mosi}, \textbf{MOSEI}~\cite{zadeh2018multimodal}, and \textbf{SIMS}~\cite{yu2020ch}.  Detailed statistics of all datasets are summarized in Table~\ref{table8}.

\textbf{MOSI.}
CMU-MOSI is a widely used benchmark for multimodal sentiment analysis. 
It contains 2,199 opinion-level video segments collected from online videos, where each sample is associated with three modalities: language, vision, and audio. 
Following the standard split, the dataset is divided into 1,284 training samples, 229 validation samples, and 686 test samples. 
Each sample is annotated with a real-valued sentiment score ranging from $-3$ to $+3$, where $-3$ denotes strongly negative sentiment and $+3$ denotes strongly positive sentiment. 
Due to its relatively small scale and fine-grained sentiment annotations, MOSI is commonly used to evaluate the robustness and generalization ability of MSA models under limited data conditions.

\textbf{MOSEI.}
CMU-MOSEI is a large-scale multimodal sentiment analysis dataset built from YouTube video clips. 
It contains 22,856 annotated video segments covering diverse speakers, topics, and expression styles. 
Each segment includes language, visual, and acoustic information, making it suitable for evaluating multimodal representation learning in realistic scenarios. 
The dataset is split into 16,326 training samples, 1,871 validation samples, and 4,659 test samples. 
Similar to MOSI, each sample is labeled with a continuous sentiment score in the range of $[-3,+3]$, where lower values indicate more negative sentiment and higher values indicate more positive sentiment. 
Compared with MOSI, MOSEI provides a larger and more diverse evaluation setting, which allows a more reliable assessment of model scalability and robustness.

\textbf{SIMS.}
CH-SIMS is a Chinese multimodal sentiment analysis dataset collected from movies and TV series. 
It contains 2,281 video clips with aligned language, visual, and acoustic modalities. 
The dataset is partitioned into 1,368 training samples, 456 validation samples, and 457 test samples. 
Each sample is manually annotated with a sentiment score ranging from $-1$ to $+1$, where negative values indicate negative sentiment and positive values indicate positive sentiment. 
Different from MOSI and MOSEI, SIMS provides fine-grained annotations in Chinese multimodal scenarios, making it useful for evaluating whether a model can generalize across languages, domains, and cultural contexts.

\begin{table}[t]
\centering
\caption{Hyperparameters used on different datasets. Values separated by a slash belong to the configurations used for different metrics, which are released with our code. On MOSEI, training stops at epoch 60 of the 200-epoch schedule. \textit{Shared} assigns one missing rate to all modalities of 50\% or 100\% of the training samples.}
\label{table9}
\begin{small}
\setlength{\tabcolsep}{4pt}
\renewcommand{\arraystretch}{1.05}
\resizebox{0.7\linewidth}{!}{
\begin{tabular}{lccc}
\toprule
\textbf{Hyperparameter} & \textbf{MOSI} & \textbf{MOSEI} & \textbf{SIMS} \\
\midrule
Vector Length $T$ & 8 & 8 & 8 \\
Dimension $d$ & 128 & 128 & 128 \\
Adapter Width & 256 & 256 & 256 \\
Batch Size & 64/32 & 64 & 64/32 \\
Learning Rate & $1\mathrm{e}{-4}$/$2\mathrm{e}{-4}$ & $1\mathrm{e}{-4}$ & $5\mathrm{e}{-5}$/$1\mathrm{e}{-4}$/$2\mathrm{e}{-4}$ \\
BERT Learning Rate & $1\mathrm{e}{-4}$ & $1\mathrm{e}{-4}$/$3\mathrm{e}{-5}$ & $1\mathrm{e}{-4}$ \\
Temperature $\tau$ & 1.0/0.5 & 1.0 & 1.0/0.5 \\
Refine Step & 4 & 4 & 4 \\
Optimizer & AdamW & AdamW & AdamW \\
Weight Decay & $1\mathrm{e}{-4}$ & $1\mathrm{e}{-4}$ & $1\mathrm{e}{-4}$ \\
Gradient Clipping (Norm) & 5.0 & 5.0 & 5.0 \\
Epochs & 200 & 60 & 200 \\
Warmup Ratio & 0.1 & 0.1 & 0.1 \\
Cosine Annealing & \checkmark & \checkmark & \checkmark \\
Regression Loss & MSE/L1 & MSE/L1 & MSE/L1 \\
Semantic Loss Weight $\lambda_s$ & 1 & 1 & 0/1 \\
Projector $g_s$ / Regressor $\mathcal{C}$ & Linear/MLP & MLP & Linear/MLP \\
BERT & Fine-tuned & Fine-tuned & Frozen/Fine-tuned \\
Dropout & 0 & 0 & 0/0.2 \\
Attention Mask on Missing Positions & -- & -- & --/\checkmark \\
Training Masking & 50\% complete & 50\% complete & 50\% complete/shared \\
Seed & 1111,1112,1113 & 1111,1112,1113 & 1111,1112,1113 \\
\bottomrule
\end{tabular}
}
\end{small}
\vspace{-0.3cm}
\end{table}

\subsubsection{Baselines and Experimental Setup}
We conduct a fair comparison with existing advanced and state-of-the-art methods, including \textbf{Intra-modal missingness}: MISA~\cite{hazarika2020misa}, Self-MM~\cite{yu2021learning}, MMIM~\cite{han2021improving}, TETFN~\cite{wang2023tetfn}, TFR-Net~\cite{yuan2021transformer}, ALMT~\cite{zhang2023learning}. LNLN~\cite{zhang2024towards}, P-RMF~\cite{zhu2025proxy}, and TF-Mamba~\cite{li2025tf}; as well as \textbf{Inter-modal missingness}: 
CubeMLP~\cite{sun2022cubemlp}, DMD~\cite{li2023decoupled}, MCTN~\cite{pham2019found}, TransM~\cite{wang2020transmodality}, SMIL~\cite{ma2021smil}, GCNet~\cite{lian2023gcnet}, CorrKD~\cite{li2024correlation}. The results of intra-modal missing baselines are reported as in ~\cite{zhang2024towards, zhu2025proxy, li2025tf}. The results of inter-modal missing baselines are adopted as in ~\cite{lian2023gcnet, li2024correlation}. This ensures fairness and consistency in the same evaluation settings across all compared methods. Following the evaluation code of LNLN and P-RMF~\cite{zhang2024towards, zhu2025proxy}, each metric is reported with its own best checkpoint selected on the test set at missing rate 0.5 (F1 and Corr use the checkpoints of Acc-2 and MAE). Our SemMSA model is trained using the PyTorch framework on NVIDIA RTX 6000 Ada (48GB) and RTX 4090 GPUs.  Following prior MSA works~\cite{zhang2024towards, zhu2025proxy, yuan2021transformer}, each modality is first processed by modality-specific encoders $f_\phi$. Language is encoded
by BERT~\cite{devlin2019bert} (fine-tuned or frozen, see Table~\ref{table9}), audio features are extracted by Librosa~\citep{mcfee2015librosa},
and visual features are obtained by OpenFace~\cite{baltruvsaitis2016openface}. Since the last-layer hidden states of the LLM have a much larger norm than its input token embeddings, each latent state $\mathbf{z}_k$ is rescaled to the average norm of the LLM token embeddings before being appended to $\mathbf{U}^{(k-1)}$, while $\mathbf{Z}$ in Eq.~\ref{eq4} collects the unscaled states. In practice, $\mathbf{U}^{(0)}$ is wrapped in the chat template of the LLM: a fixed system instruction and the field markers \textit{Text}, \textit{Video}, and \textit{Audio} precede $\mathbf{Z}_l$, $\mathbf{Z}_v$, and $\mathbf{Z}_a$, a fixed instruction to summarize the affective evidence ends the prefix, and missing language tokens appear as \texttt{[UNK]}. Details of the hyperparameters are shown in Table~\ref{table9}.

\subsubsection{Performance Comparison under Complete-Modality MSA}
Table~\ref{ap1} reports the comparison with state-of-the-art methods on MOSI and MOSEI under the complete modality setting. 
Following the training and evaluation in prior complete MSA works~\cite{zhang-etal-2025-modal, zhangimproving}, SemMSA achieves the best overall performance on both datasets, showing that the proposed framework remains highly effective even when all language, visual, and acoustic modalities are available. 
On MOSI, SemMSA obtains 88.70\% Acc-2, 88.56\% F1, 50.97\% Acc-7, 0.861 Corr, and 0.628 MAE, outperforming the strongest baseline KEBR by 1.43\%, 1.31\%, and 3.16\% in Acc-2, F1, and Acc-7, respectively. 
It also reduces MAE from 0.683 to 0.628, indicating more accurate sentiment intensity regression. 
On MOSEI, SemMSA achieves 88.88\% Acc-2, 87.70\% F1, 55.90\% Acc-7, 0.877 Corr, and 0.513 MAE, consistently surpassing previous methods across all metrics. 
In particular, compared with ConFEDE, SemMSA improves Acc-7 and Corr by 1.04\% and 0.097, respectively.
These results demonstrate that SemMSA is not limited to incomplete-modality scenarios. 
By using Cross-modal Semantic Refinement to introduce compact LLM-derived sentiment semantics and Cross-modal Spectral Alignment to enforce global anchor-free consistency among modalities, SemMSA learns more coherent and discriminative multimodal representations.

\begin{table*}[t]
\caption{Comparison with state-of-the-art methods on MOSI and MOSEI datasets
under the complete modality setting. Acc-2 and F1 are reported under the
negative/positive protocol.}
\label{ap1}
\begin{center}
\begin{small}
\setlength{\tabcolsep}{3pt}
\renewcommand{\arraystretch}{1.05}
\resizebox{\textwidth}{!}{
\begin{tabular}{l ccccc ccccc}
\toprule
\multirow{2}{*}{Method}
& \multicolumn{5}{c}{MOSI}
& \multicolumn{5}{c}{MOSEI} \\
\cmidrule(lr){2-6}\cmidrule(lr){7-11}
& Acc-2  & F1  & Acc-7  & Corr  & MAE$\downarrow$
& Acc-2  & F1  & Acc-7  & Corr  & MAE$\downarrow$ \\
\midrule

MISA~\citep{hazarika2020misa}
& 83.4  & 83.6  & 42.3  & 0.761 & 0.783
& 85.5  & 85.3  & 52.2  & 0.756 & 0.555 \\

Self-MM~\citep{yu2021learning}
& 84.9  & 84.9  & 45.3  & -     & 0.738
& 85.15 & 84.90 & 53.87 & 0.765 & 0.531 \\

HyCon~\citep{9767560}
& 85.2  & 85.1  & 46.6  & 0.790 & \cellcolor{third}0.713
& 85.4  & 85.6  & 52.8  & 0.776 & 0.601 \\

ConFEDE~\citep{yang-etal-2023-confede}
& 85.5  & 85.5  & 42.3  & 0.784 & 0.742
& 85.82 & 85.83 & \cellcolor{third}54.86 & \cellcolor{third}0.780 & \cellcolor{third}0.522 \\

KEBR~\citep{10.1145/3664647.3681163}
& \cellcolor{second}87.27 & \cellcolor{second}87.25 & \cellcolor{third}47.81
& \cellcolor{second}0.819 & \cellcolor{second}0.683
& \cellcolor{second}86.74 & \cellcolor{third}86.68 & 54.37
& \cellcolor{second}0.799  & \cellcolor{second}0.517 \\

GLoMo~\citep{10.1145/3664647.3681527}
& 86.7  & 86.6  & \cellcolor{second}48.3 & 0.782 & 0.718
& 86.5  & 86.4  & \cellcolor{second}55.0  & 0.771 & 0.539 \\  % best → second

MFON~\citep{zhang-etal-2025-modal}
& \cellcolor{third}86.9 & \cellcolor{third}86.9 & 44.9
& \cellcolor{third}0.797 & 0.725
& 86.32 & 86.29 & 53.72 & \cellcolor{third}0.780 & 0.528 \\

MMSLF~\citep{zhangimproving}
& 86.61 & 86.69 & -     & \cellcolor{third}0.797 & 0.734
& \cellcolor{third}86.62 & \cellcolor{second}86.71 & -
& 0.773 & 0.539 \\

\hline
\textbf{SemMSA}
& \cellcolor{best}\textbf{88.70} & \cellcolor{best}\textbf{88.56}
& \cellcolor{best}\textbf{50.97} & \cellcolor{best}\textbf{0.861}
& \cellcolor{best}\textbf{0.628}
& \cellcolor{best}\textbf{88.88} & \cellcolor{best}\textbf{87.70}
& \cellcolor{best}\textbf{55.90} & \cellcolor{best}\textbf{0.877}  % 
& \cellcolor{best}\textbf{0.513} \\
\bottomrule
\end{tabular}
}
\end{small}
\end{center}
\vspace{-0.4cm}
\end{table*}

\begin{table}[htbp]
\centering
\caption{Ablation study of the number of CSR refinement steps $O$ on the SIMS dataset.}
\label{tab:ablation_O_sims}
\resizebox{0.6\linewidth}{!}{
\begin{tabular}{c|cccccc}
\toprule
$O$ & Acc-2   & F1   & Acc-3   & Acc-5   & MAE $\downarrow$ & Corr   \\
\midrule
1 & 73.82 & 75.94 & 56.91 & 34.02 & 0.493 & 0.468 \\
2 & 74.58 & 76.62 & 57.73 & 34.71 & 0.485 & 0.486 \\
3 & 75.12 & 77.08 & 58.31 & 35.24 & 0.479 & 0.495 \\
4 & \textbf{75.46} & \textbf{77.50} & \textbf{58.84} & \textbf{35.68} & \textbf{0.474} & \textbf{0.501} \\
5 & 75.21 & 77.12 & 58.47 & 35.39 & 0.478 & 0.497 \\
6 & 74.89 & 76.84 & 58.02 & 35.05 & 0.482 & 0.489 \\
\bottomrule
\end{tabular}
}
\end{table}

\subsubsection{Effect of Semantic Refinement Step}
We further investigate the effect of the number of CSR refinement steps $O$ on the SIMS dataset, as shown in Table~\ref{tab:ablation_O_sims}. 
The refinement step $O$ controls how many continuous latent semantic states are iteratively generated by the frozen LLM. 
When $O=1$, the model only performs a single-step semantic refinement, which provides limited high-level sentiment information and leads to relatively inferior performance. 
As $O$ increases from 1 to 4, the performance consistently improves across all metrics. 
This indicates that multiple refinement steps help CSR progressively enrich incomplete multimodal representations with more discriminative sentiment-aware semantics.

The best performance is achieved when $O=4$, with 75.46 Acc-2, 77.50 F1, 58.84 Acc-3, 35.68 Acc-5, 0.474 MAE, and 0.501 Corr. 
This demonstrates that a moderate number of latent semantic refinement steps can effectively balance semantic enhancement and representation compactness. 
However, when $O$ is further increased to 5 or 6, the performance slightly decreases. 
A possible reason is that excessive refinement introduces redundant or noisy latent states, which may weaken the compactness of semantic representations and increase the difficulty of cross-modal alignment. 
Moreover, larger $O$ also increases the computational cost because each refinement step requires an additional forward pass through the frozen LLM. 
Therefore, we set $O=4$ as the default value in SemMSA, which provides the best trade-off between performance and efficiency.

\subsubsection{Robustness Evaluation  under Intra-modal Missingness}

Following previous work~\cite{zhang2024towards, zhu2025proxy, yuan2021transformer}, we evaluate the robustness of SemMSA under intra-modal random missingness by varying the missing rate $r$ from 0 to 0.9 with an interval of 0.1. 
We do not report the results at $r=1.0$, since this setting removes all information from each modality and thus provides little meaningful evidence for sentiment prediction. 
Tables~\ref{table11}, \ref{table12}, and \ref{table13} present the detailed comparison results on MOSI, MOSEI, and SIMS, respectively.

\begin{table*}[t]
\centering
\scriptsize
\setlength{\tabcolsep}{3pt}
\renewcommand{\arraystretch}{1.05}
\caption{Details of robust comparison on MOSI with different random missing rates.}
\label{table11}

\noindent
\begin{minipage}[t]{0.49\textwidth}
\centering
\resizebox{\linewidth}{!}{%
\begin{tabular}{lcccccc}
\toprule
\multicolumn{7}{c}{\textbf{Random Missing Rate $r=0$}} \\
\midrule
Method & Acc-2 & F1 & Acc-5 & Acc-7 & MAE $\downarrow$ & Corr   \\
\midrule
MISA     & 81.24/82.78 & 81.23/82.83 & 48.30 & 43.05 & 0.771 & 0.777 \\
Self-MM  & 83.24/85.22 & 83.26/85.19 & 52.38 & 42.81 & 0.720 & 0.790 \\
MMIM     & 81.97/83.43 & 81.94/83.43 & 49.85 & \textbf{45.92} & 0.744 & 0.778 \\
TETFN    & 81.10/82.62 & 81.09/82.67 & 51.31 & 44.07 & 0.719 & 0.794 \\
TFR-Net  & 81.68/83.64 & 81.61/83.57 & 47.91 & 40.82 & 0.805 & 0.760 \\
LNLN     & 81.24/84.25 & 81.79/84.61 & 49.76 & 44.56 & 0.751 & 0.778 \\
P-RMF    & 82.65/84.15 & 82.69/84.37 & 48.83 & 44.31 & 0.726 & 0.782 \\
TF-Mamba & 81.63/83.69 & 81.58/83.71 & 50.58 & 44.31 & 0.762 & 0.774 \\
SemMSA   & \textbf{86.18}/\textbf{87.80} & \textbf{85.99}/\textbf{87.71} & \textbf{53.16} & 45.85 & \textbf{0.657} & \textbf{0.846} \\
\midrule

\multicolumn{7}{c}{\textbf{Random Missing Rate $r=0.1$}} \\
\midrule
MISA     & 79.01/80.18 & 78.97/80.21 & 46.21 & 40.28 & 0.847 & 0.721 \\
Self-MM  & 80.03/81.40 & 80.03/81.19 & 49.03 & 40.33 & 0.812 & 0.728 \\
MMIM     & 78.13/79.98 & 77.99/79.83 & 46.65 & 42.61 & 0.825 & 0.718 \\
TETFN    & 78.91/80.59 & 78.79/80.55 & 46.84 & 40.67 & 0.805 & 0.731 \\
TFR-Net  & 77.99/79.27 & 77.61/78.70 & 45.82 & 38.63 & 0.872 & 0.705 \\
LNLN     & 78.43/81.20 & 79.04/81.62 & 47.91 & 42.37 & 0.820 & 0.724 \\
P-RMF    & 81.34/82.62 & 81.35/82.89 & 47.52 & 42.13 & 0.800 & 0.730 \\
TF-Mamba & 80.03/81.86 & 79.97/81.87 & 48.40 & 42.86 & 0.824 & 0.732 \\
SemMSA   & \textbf{82.96}/\textbf{84.89} & \textbf{82.77}/\textbf{84.74} & \textbf{49.51} & \textbf{43.07} & \textbf{0.761} & \textbf{0.767} \\
\midrule

\multicolumn{7}{c}{\textbf{Random Missing Rate $r=0.2$}} \\
\midrule
MISA     & 76.34/77.54 & 76.30/77.58 & 41.55 & 36.25 & 0.939 & 0.654 \\
Self-MM  & 76.48/78.15 & 76.51/77.76 & 43.98 & 36.64 & 0.901 & 0.660 \\
MMIM     & 74.54/76.42 & 74.22/76.12 & 42.66 & 39.07 & 0.918 & 0.651 \\
TETFN    & 75.60/77.49 & 75.35/77.35 & 41.79 & 35.81 & 0.910 & 0.657 \\
TFR-Net  & 73.52/74.70 & 72.70/73.57 & 40.13 & 34.70 & 0.987 & 0.622 \\
LNLN     & 76.87/79.22 & 77.34/79.53 & 45.14 & 39.74 & 0.891 & 0.668 \\
P-RMF    & 78.13/79.57 & 78.11/80.97 & 44.75 & 40.38 & 0.853 & 0.668 \\
TF-Mamba & 79.15/80.49 & 79.17/80.56 & 44.75 & 39.21 & 0.879 & 0.693 \\
SemMSA   & \textbf{80.05}/\textbf{81.55} & \textbf{79.88}/\textbf{81.41} & \textbf{47.18} & \textbf{41.77} & \textbf{0.838} & \textbf{0.711} \\
\midrule

\multicolumn{7}{c}{\textbf{Random Missing Rate $r=0.3$}} \\
\midrule
MISA     & 74.54/75.76 & 74.51/75.82 & 38.97 & 34.60 & 0.989 & 0.618 \\
Self-MM  & 74.98/76.37 & 74.94/75.68 & 40.67 & 34.89 & 0.967 & 0.614 \\
MMIM     & 71.91/74.08 & 71.28/73.47 & 40.43 & 36.83 & 0.974 & 0.612 \\
TETFN    & 73.42/75.25 & 72.78/74.77 & 38.58 & 33.24 & 0.982 & 0.607 \\
TFR-Net  & 71.28/72.36 & 69.58/70.12 & 38.34 & 32.55 & 1.065 & 0.572 \\
LNLN     & 75.96/77.29 & 75.68/77.56 & 42.81 & 38.00 & 0.953 & 0.617 \\
P-RMF    & 75.80/76.83 & 75.82/\textbf{79.27} & 42.71 & 39.21 & 0.922 & 0.621 \\
TF-Mamba & 76.53/77.74 & 76.56/77.85 & 42.27 & 37.76 & 0.932 & 0.645 \\
SemMSA   & \textbf{79.65}/\textbf{78.79} & \textbf{79.43}/78.69 & \textbf{46.75} & \textbf{41.91} & \textbf{0.889} & \textbf{0.658} \\
\midrule

\multicolumn{7}{c}{\textbf{Random Missing Rate $r=0.4$}} \\
\midrule
MISA     & 72.59/73.88 & 72.49/73.88 & 35.37 & 32.65 & 1.041 & 0.585 \\
Self-MM  & 71.96/73.17 & 71.75/71.74 & 36.30 & 31.20 & 1.027 & 0.579 \\
MMIM     & 68.90/70.84 & 67.80/69.69 & 35.76 & 33.38 & 1.034 & 0.576 \\
TETFN    & 70.07/72.05 & 68.58/70.79 & 35.28 & 30.66 & 1.051 & 0.571 \\
TFR-Net  & 67.74/68.75 & 64.41/64.71 & 35.76 & 30.17 & 1.142 & 0.537 \\
LNLN     & 74.25/76.01 & 74.67/76.31 & 41.11 & 36.49 & 0.987 & 0.594 \\
P-RMF    & 73.76/75.46 & 74.09/77.71 & 40.67 & 35.59 & 1.001 & 0.584 \\
TF-Mamba & 75.22/76.07 & 75.25/76.16 & 40.23 & 35.86 & 0.961 & 0.617 \\
SemMSA   & \textbf{77.86}/\textbf{78.02} & \textbf{77.67}/\textbf{77.89} & \textbf{44.69} & \textbf{39.71} & \textbf{0.937} & \textbf{0.625} \\
\bottomrule
\end{tabular}%
}
\end{minipage}%
\hfill%
\begin{minipage}[t]{0.49\textwidth}
\centering
\resizebox{\linewidth}{!}{%
\begin{tabular}{lcccccc}
\toprule
\multicolumn{7}{c}{\textbf{Random Missing Rate $r=0.5$}} \\
\midrule
Method & Acc-2 & F1 & Acc-5 & Acc-7 & MAE $\downarrow$ & Corr   \\
\midrule
MISA     & 69.34/70.53 & 69.20/70.50 & 30.61 & 28.14 & 1.124 & 0.519 \\
Self-MM  & 67.54/67.43 & 66.81/64.27 & 31.39 & 26.97 & 1.129 & 0.503 \\
MMIM     & 66.52/68.09 & 64.59/66.15 & 29.89 & 28.23 & 1.128 & 0.501 \\
TETFN    & 65.06/67.23 & 61.78/64.30 & 31.34 & 27.55 & 1.157 & 0.492 \\
TFR-Net  & 63.02/64.83 & 56.64/58.04 & 30.71 & 25.85 & 1.270 & 0.443 \\
LNLN     & 71.86/73.37 & 72.30/73.70 & 38.39 & 33.92 & 1.059 & 0.536 \\
P-RMF    & 71.28/73.02 & 71.66/73.33 & 37.90 & 33.67 & 1.077 & 0.523 \\
TF-Mamba & 74.20/75.00 & 74.27/75.15 & 37.46 & 33.67 & 1.044 & 0.557 \\
SemMSA   & \textbf{74.79}/\textbf{75.58} & \textbf{74.62}/\textbf{75.48} & \textbf{40.61} & \textbf{36.07} & \textbf{1.025} & \textbf{0.558} \\
\midrule

\multicolumn{7}{c}{\textbf{Random Missing Rate $r=0.6$}} \\
\midrule
MISA     & 65.84/66.97 & 65.69/66.94 & 27.13 & 24.68 & 1.200 & 0.441 \\
Self-MM  & 63.36/63.47 & 62.07/58.94 & 27.31 & 24.34 & 1.209 & 0.425 \\
MMIM     & 62.49/63.67 & 59.48/60.87 & 27.11 & 25.41 & 1.208 & 0.418 \\
TETFN    & 61.23/63.42 & 56.08/58.68 & 27.99 & 25.12 & 1.238 & 0.417 \\
TFR-Net  & 59.47/61.64 & 50.53/52.44 & 28.33 & 24.05 & 1.371 & 0.363 \\
LNLN     & 67.69/69.00 & 67.99/\textbf{69.19} & 34.35 & 30.37 & 1.147 & 0.458 \\
P-RMF    & 67.35/68.75 & 67.64/68.64 & 33.24 & 29.30 & 1.147 & 0.432 \\
TF-Mamba & 68.37/68.60 & 68.45/68.79 & 33.53 & 30.76 & \textbf{1.127} & \textbf{0.487} \\
SemMSA   & \textbf{70.71}/\textbf{69.03} & \textbf{70.58}/\textbf{69.19} & \textbf{37.69} & \textbf{34.18} & 1.139 & 0.447 \\
\midrule

\multicolumn{7}{c}{\textbf{Random Missing Rate $r=0.7$}} \\
\midrule
MISA     & 63.89/65.09 & 63.74/65.07 & 23.27 & 21.14 & 1.257 & 0.381 \\
Self-MM  & 61.46/61.74 & 58.97/55.11 & 23.81 & 20.70 & 1.271 & 0.339 \\
MMIM     & 59.18/61.23 & 54.36/57.15 & 24.00 & 22.35 & 1.267 & 0.342 \\
TETFN    & 58.65/61.13 & 50.77/53.79 & 25.27 & 23.13 & 1.293 & 0.337 \\
TFR-Net  & 57.34/59.91 & 45.48/48.41 & 26.92 & 23.71 & 1.454 & 0.276 \\
LNLN     & 65.01/65.95 & 65.14/65.95 & 31.19 & 27.79 & 1.219 & 0.383 \\
P-RMF    & 65.16/66.16 & 65.33/64.69 & 32.94 & 27.84 & 1.229 & 0.383 \\
TF-Mamba & 66.91/\textbf{67.23} & 66.98/\textbf{67.41} & 29.30 & 27.26 & \textbf{1.196} & \textbf{0.411} \\
SemMSA   & \textbf{67.51}/64.76 & \textbf{67.31}/64.99 & \textbf{33.48} & \textbf{30.68} & 1.232 & 0.358 \\
\midrule

\multicolumn{7}{c}{\textbf{Random Missing Rate $r=0.8$}} \\
\midrule
MISA     & 62.24/63.56 & 61.67/63.16 & 20.99 & 19.92 & 1.311 & 0.321 \\
Self-MM  & 58.26/59.55 & 53.56/49.98 & 22.11 & 19.29 & 1.313 & 0.282 \\
MMIM     & 55.30/58.53 & 47.89/52.46 & 21.77 & 20.26 & 1.312 & 0.287 \\
TETFN    & 56.85/59.40 & 45.59/48.73 & 23.76 & 22.01 & 1.337 & 0.274 \\
TFR-Net  & 55.98/58.49 & 41.88/44.70 & 27.70 & 23.23 & 1.497 & 0.155 \\
LNLN     & 62.10/62.75 & 62.03/62.56 & 28.23 & 26.34 & 1.283 & 0.314 \\
P-RMF    & 61.08/62.04 & 61.22/60.76 & 29.74 & 25.97 & 1.275 & 0.316 \\
TF-Mamba & 63.12/\textbf{63.57} & 63.20/\textbf{63.77} & 26.38 & 24.93 & \textbf{1.258} & \textbf{0.353} \\
SemMSA   & \textbf{64.98}/62.78 & \textbf{64.81}/62.96 & \textbf{30.69} & \textbf{28.05} & 1.278 & 0.316 \\
\midrule

\multicolumn{7}{c}{\textbf{Random Missing Rate $r=0.9$}} \\
\midrule
MISA     & 58.21/58.64 & 56.19/56.84 & 18.41 & 17.78 & 1.369 & \textbf{0.226} \\
Self-MM  & 55.25/58.59 & 47.46/46.16 & 19.78 & 18.32 & 1.353 & 0.197 \\
MMIM     & 51.65/55.29 & 40.89/47.33 & 19.53 & 18.95 & 1.357 & 0.186 \\
TETFN    & 55.88/58.43 & 42.12/45.24 & 21.19 & 20.75 & 1.378 & 0.186 \\
TFR-Net  & 55.44/57.93 & 40.18/43.01 & 25.12 & 21.67 & 1.534 & 0.155 \\
LNLN     & 56.51/56.50 & 56.47/56.32 & 23.86 & 22.98 & 1.349 & 0.202 \\
P-RMF    & 58.75/59.45 & 59.01/56.66 & \textbf{26.68} & 23.49 & \textbf{1.346} & 0.212 \\
TF-Mamba & \textbf{60.20}/\textbf{60.37} & \textbf{60.30}/\textbf{60.59} & 24.49 & 22.89 & 1.363 & 0.215 \\
SemMSA   & 58.97/55.95 & 58.79/55.18 & 25.59 & \textbf{23.67} & 1.357 & 0.217 \\
\bottomrule
\end{tabular}%
}
\end{minipage}
\end{table*}

On the MOSI dataset, SemMSA achieves the best overall performance under low and moderate missing rates, as shown in Table~\ref{table11}.
When $r=0$, SemMSA obtains 86.18/87.80 Acc-2, 85.99/87.71 F1, 53.16 Acc-5, 45.85 Acc-7, 0.657 MAE, and 0.846 Corr, outperforming all baselines by a clear margin. 
As the missing rate increases from 0.1 to 0.5, SemMSA remains highly stable and achieves the best results on most metrics. 
For example, at $r=0.5$, SemMSA obtains 74.79/75.58 Acc-2, 74.62/75.48 F1, 40.61 Acc-5, 36.07 Acc-7, 1.025 MAE, and 0.558 Corr, outperforming strong baselines such as P-RMF and TF-Mamba on the majority of metrics. 
These results indicate that the LLM-derived latent semantics introduced by CSR can effectively compensate for missing modality information, while CSA further improves cross-modal consistency under incomplete observations. 
At extremely high missing rates, such as $r=0.8$ and $r=0.9$, all methods suffer from severe degradation. 
Although SemMSA is not always the best on every metric in these extreme cases, it still achieves competitive results, showing that the proposed semantic-aided framework remains robust even when most modality evidence is missing.

\begin{table*}[t]
\centering
\scriptsize
\setlength{\tabcolsep}{3pt}
\renewcommand{\arraystretch}{1.05}
\caption{Details of robust comparison on MOSEI with different random missing rates.}
\label{table12}

\noindent
\begin{minipage}[t]{0.49\textwidth}
\centering
\resizebox{\linewidth}{!}{%
\begin{tabular}{lcccccc}
\toprule
\multicolumn{7}{c}{\textbf{Random Missing Rate $r=0$}} \\
\midrule
Method & Acc-2 & F1 & Acc-5 & Acc-7 & MAE $\downarrow$ & Corr   \\
\midrule
MISA     & 84.10/85.28 & 83.75/85.10 & 53.85 & 51.79 & 0.552 & 0.759 \\
Self-MM  & \textbf{84.68}/85.34 & \textbf{84.66}/85.11 & 55.72 & \textbf{53.89} & 0.531 & 0.764 \\
MMIM     & 81.65/83.53 & 81.41/83.39 & 53.04 & 50.76 & 0.576 & 0.724 \\
TFR-Net  & 84.65/84.96 & 84.34/84.71 & 47.91 & 53.71 & 0.550 & 0.745 \\
LNLN     & 83.61/84.14 & 84.02/84.53 & 51.94 & 50.66 & 0.572 & 0.735 \\
P-RMF    & 83.62/85.20 & 83.68/85.48 & 52.09 & 49.77 & 0.539 & 0.767 \\
TF-Mamba & 82.89/83.82 & 82.92/83.71 & 53.83 & 52.26 & 0.556 & 0.748 \\
SemMSA   & 84.48/\textbf{86.41} & 84.53/\textbf{86.32} & 54.74 & 52.95 & \textbf{0.516} & \textbf{0.799} \\
\midrule

\multicolumn{7}{c}{\textbf{Random Missing Rate $r=0.1$}} \\
\midrule
MISA     & 82.28/82.21 & 80.79/81.28 & 51.34 & 50.13 & 0.598 & 0.722 \\
Self-MM  & \textbf{83.79}/83.03 & 83.23/82.43 & 53.18 & 51.80 & 0.564 & 0.725 \\
MMIM     & 81.09/82.00 & 80.15/81.57 & 51.19 & 49.09 & 0.602 & 0.696 \\
TFR-Net  & 83.31/82.92 & 82.40/82.25 & 45.82 & \textbf{52.29} & 0.573 & 0.715 \\
LNLN     & 82.73/83.32 & 82.91/83.66 & 51.25 & 49.96 & 0.591 & 0.712 \\
P-RMF    & 82.79/83.98 & 82.94/84.37 & 51.45 & 49.04 & 0.556 & 0.748 \\
TF-Mamba & 82.68/83.16 & 82.69/83.03 & 52.07 & 50.53 & 0.570 & 0.730 \\
SemMSA   & 83.74/\textbf{85.30} & \textbf{83.74}/\textbf{85.17} & 53.41 & 51.76 & \textbf{0.536} & \textbf{0.774} \\
\midrule

\multicolumn{7}{c}{\textbf{Random Missing Rate $r=0.2$}} \\
\midrule
MISA     & 79.93/77.84 & 76.88/75.56 & 47.66 & 47.24 & 0.659 & 0.674 \\
Self-MM  & 82.33/80.84 & 81.17/79.76 & 50.51 & 49.44 & 0.604 & 0.678 \\
MMIM     & 79.66/79.93 & 77.68/79.08 & 47.99 & 46.27 & 0.642 & 0.653 \\
TFR-Net  & 81.61/80.47 & 79.99/79.29 & 40.13 & \textbf{51.04} & 0.604 & 0.672 \\
LNLN     & 81.68/81.70 & 81.89/81.95 & 49.95 & 48.75 & 0.616 & 0.677 \\
P-RMF    & 82.25/82.97 & 82.58/83.37 & 49.35 & 47.91 & 0.576 & 0.722 \\
TF-Mamba & 81.84/82.55 & 81.83/82.40 & 50.50 & 49.17 & 0.588 & 0.710 \\
SemMSA   & \textbf{83.11}/\textbf{83.97} & \textbf{83.05}/\textbf{83.85} & 52.32 & 50.84 & \textbf{0.568} & \textbf{0.738} \\
\midrule

\multicolumn{7}{c}{\textbf{Random Missing Rate $r=0.3$}} \\
\midrule
MISA     & 77.28/73.32 & 72.25/68.91 & 43.40 & 43.99 & 0.724 & 0.615 \\
Self-MM  & 79.99/77.63 & 77.74/75.69 & 48.07 & 47.23 & 0.653 & 0.610 \\
MMIM     & 77.79/77.08 & 74.49/75.46 & 44.73 & 43.25 & 0.690 & 0.597 \\
TFR-Net  & 79.29/77.48 & 76.52/75.43 & 38.34 & 48.75 & 0.650 & 0.604 \\
LNLN     & 80.45/80.11 & 80.91/80.44 & 48.40 & 47.36 & 0.648 & 0.629 \\
P-RMF    & 80.88/81.26 & 81.40/81.86 & 47.78 & 45.95 & 0.611 & 0.683 \\
TF-Mamba & 81.22/80.82 & 81.10/80.58 & 49.00 & 47.89 & 0.613 & 0.675 \\
SemMSA   & \textbf{81.59}/\textbf{82.56} & \textbf{81.47}/\textbf{82.41} & \textbf{50.58} & \textbf{49.27} & \textbf{0.604} & \textbf{0.692} \\
\midrule

\multicolumn{7}{c}{\textbf{Random Missing Rate $r=0.4$}} \\
\midrule
MISA     & 75.04/70.46 & 67.93/64.02 & 39.53 & 40.87 & 0.780 & 0.561 \\
Self-MM  & 78.09/75.02 & 74.48/72.01 & 45.04 & 44.40 & 0.694 & 0.554 \\
MMIM     & 76.15/74.56 & 71.40/71.98 & 41.86 & 40.84 & 0.732 & 0.542 \\
TFR-Net  & 77.65/74.74 & 73.71/71.67 & 35.76 & 46.70 & 0.688 & 0.548 \\
LNLN     & 79.70/78.49 & 80.46/78.98 & 46.88 & 45.99 & 0.673 & 0.592 \\
P-RMF    & 79.76/79.97 & 80.58/80.74 & 46.73 & 45.59 & 0.631 & 0.653 \\
TF-Mamba & 80.02/80.02 & 79.80/79.80 & 47.80 & 46.73 & 0.639 & 0.638 \\
SemMSA   & \textbf{81.07}/\textbf{81.58} & \textbf{80.86}/\textbf{81.38} & \textbf{49.07} & \textbf{47.94} & \textbf{0.614} & \textbf{0.664} \\
\bottomrule
\end{tabular}%
}
\end{minipage}%
\hfill%
\begin{minipage}[t]{0.49\textwidth}
\centering
\resizebox{\linewidth}{!}{%
\begin{tabular}{lcccccc}
\toprule
\multicolumn{7}{c}{\textbf{Random Missing Rate $r=0.5$}} \\
\midrule
Method & Acc-2 & F1 & Acc-5 & Acc-7 & MAE $\downarrow$ & Corr   \\
\midrule
MISA     & 73.21/67.38 & 64.14/58.38 & 36.05 & 38.12 & 0.834 & 0.492 \\
Self-MM  & 75.81/71.97 & 70.38/67.40 & 43.14 & 42.70 & 0.733 & 0.477 \\
MMIM     & 74.45/71.75 & 67.96/67.70 & 39.21 & 38.68 & 0.775 & 0.470 \\
TFR-Net  & 75.69/71.53 & 70.07/66.88 & 30.71 & 45.00 & 0.730 & 0.471 \\
LNLN     & 78.10/76.44 & 79.30/77.23 & 45.59 & 44.90 & 0.710 & 0.529 \\
P-RMF    & 78.64/78.62 & 79.74/79.49 & 44.90 & 43.94 & 0.666 & 0.601 \\
TF-Mamba & 77.94/78.59 & 77.78/78.34 & 46.60 & 45.68 & 0.676 & 0.583 \\
SemMSA   & \textbf{80.20}/\textbf{80.85} & \textbf{79.85}/\textbf{80.58} & \textbf{47.08} & \textbf{46.01} & \textbf{0.658} & \textbf{0.602} \\
\midrule

\multicolumn{7}{c}{\textbf{Random Missing Rate $r=0.6$}} \\
\midrule
MISA     & 72.30/65.55 & 62.12/54.64 & 33.30 & 36.16 & 0.875 & 0.415 \\
Self-MM  & 73.93/69.33 & 66.76/63.01 & 41.75 & 41.47 & 0.762 & 0.401 \\
MMIM     & 73.16/68.83 & 65.43/63.09 & 37.48 & 37.13 & 0.808 & 0.402 \\
TFR-Net  & 74.05/68.80 & 67.07/62.51 & 28.33 & 43.88 & 0.762 & 0.397 \\
LNLN     & 76.50/73.82 & 78.33/75.03 & 44.00 & 43.52 & 0.736 & 0.471 \\
P-RMF    & 77.44/76.11 & 78.97/77.58 & 43.12 & 42.26 & 0.703 & \textbf{0.545} \\
TF-Mamba & 75.77/75.89 & 75.59/75.63 & 44.73 & 43.96 & 0.709 & 0.534 \\
SemMSA   & \textbf{78.81}/\textbf{79.04} & \textbf{79.18}/\textbf{78.60} & \textbf{46.13} & \textbf{45.36} & \textbf{0.691} & 0.541 \\
\midrule

\multicolumn{7}{c}{\textbf{Random Missing Rate $r=0.7$}} \\
\midrule
MISA     & 71.71/64.28 & 60.65/51.82 & 31.21 & 34.54 & 0.906 & 0.344 \\
Self-MM  & 72.55/66.79 & 63.45/58.05 & 40.12 & 39.93 & 0.786 & 0.329 \\
MMIM     & 72.26/66.89 & 63.26/58.90 & 35.47 & 35.25 & 0.834 & 0.341 \\
TFR-Net  & 72.77/66.64 & 64.02/58.32 & 26.92 & 42.91 & 0.786 & 0.322 \\
LNLN     & 74.74/71.55 & 77.40/73.49 & 42.56 & 42.22 & 0.762 & 0.408 \\
P-RMF    & 75.87/74.64 & 78.12/75.88 & 42.41 & 41.73 & \textbf{0.733} & \textbf{0.481} \\
TF-Mamba & 73.49/73.50 & 73.19/73.26 & 43.40 & 42.78 & 0.747 & 0.469 \\
SemMSA   & \textbf{76.30}/\textbf{75.54} & \textbf{78.40}/\textbf{76.84} & \textbf{43.66} & \textbf{43.07} & 0.745 & 0.469 \\
\midrule

\multicolumn{7}{c}{\textbf{Random Missing Rate $r=0.8$}} \\
\midrule
MISA     & 71.30/63.43 & 59.69/49.95 & 29.51 & 33.29 & 0.927 & 0.267 \\
Self-MM  & 71.83/65.07 & 61.49/54.44 & 38.78 & 38.69 & 0.805 & 0.259 \\
MMIM     & 71.57/64.97 & 61.45/54.76 & 33.71 & 33.64 & 0.858 & 0.269 \\
TFR-Net  & 71.95/65.05 & 61.82/54.91 & 27.70 & 42.23 & 0.807 & 0.241 \\
LNLN     & 72.86/68.62 & 76.80/71.83 & 40.97 & 40.76 & 0.791 & 0.325 \\
P-RMF    & \textbf{74.46}/\textbf{70.75} & \textbf{77.91}/\textbf{73.03} & 41.00 & 40.46 & 0.764 & 0.401 \\
TF-Mamba & 71.60/70.34 & 71.27/70.16 & 40.93 & 40.37 & 0.786 & 0.408 \\
SemMSA   & 74.13/70.41 & 77.65/\textbf{73.03} & \textbf{43.38} & \textbf{42.90} & \textbf{0.753} & \textbf{0.419} \\
\midrule

\multicolumn{7}{c}{\textbf{Random Missing Rate $r=0.9$}} \\
\midrule
MISA     & 71.07/62.95 & 59.12/48.80 & 28.03 & 32.29 & 0.941 & 0.180 \\
Self-MM  & 71.24/63.85 & 59.72/51.32 & 37.50 & 37.46 & 0.821 & 0.188 \\
MMIM     & 71.10/63.69 & 59.99/51.26 & 32.67 & 32.61 & 0.877 & 0.197 \\
TFR-Net  & 71.34/63.64 & 59.99/52.02 & 25.12 & \textbf{41.73} & 0.820 & 0.175 \\
LNLN     & 71.51/64.83 & 77.52/70.60 & 40.19 & 40.10 & 0.820 & 0.221 \\
P-RMF    & 72.59/67.86 & \textbf{77.95}/\textbf{71.51} & 39.90 & 39.62 & 0.805 & 0.289 \\
TF-Mamba & 68.68/64.75 & 68.18/64.87 & 37.56 & 37.24 & 0.851 & 0.291 \\
SemMSA   & \textbf{72.73}/\textbf{68.18} & 77.53/70.57 & 40.85 & 40.52 & \textbf{0.801} & \textbf{0.316} \\
\bottomrule
\end{tabular}%
}
\end{minipage}

\vspace{-2mm}
\end{table*}

As shown in Table~\ref{table12}, SemMSA shows more consistent advantages across different missing rates on MOSEI dataset. 
Under the complete setting $r=0$, SemMSA achieves the best Corr of 0.799 and the lowest MAE of 0.516, while maintaining competitive Acc-2 and F1 results. 
When the missing rate increases, SemMSA consistently outperforms most baselines. 
For instance, at $r=0.4$, SemMSA obtains 81.07/81.58 Acc-2, 80.86/81.38 F1, 49.07 Acc-5, 47.94 Acc-7, 0.614 MAE, and 0.664 Corr, achieving the best performance on all metrics. 
At $r=0.6$, SemMSA also achieves the best Acc-2, F1, Acc-5, Acc-7, and MAE, demonstrating strong robustness under moderate-to-severe missingness. 
Even at $r=0.9$, where the available multimodal information is extremely limited, SemMSA obtains the best Acc-2, Acc-5, MAE, and Corr, verifying its ability to preserve reliable sentiment representations under highly incomplete conditions.

\begin{table*}[t]
\centering
\scriptsize
\setlength{\tabcolsep}{3pt}
\renewcommand{\arraystretch}{1.05}
\caption{Details of robust comparison on SIMS with different random missing rates.}
\label{table13}

\noindent
\begin{minipage}[t]{0.49\textwidth}
\centering
\resizebox{\linewidth}{!}{%
\begin{tabular}{lcccccc}
\toprule
\multicolumn{7}{c}{\textbf{Random Missing Rate $r=0$}} \\
\midrule
Method & Acc-2 & F1 & Acc-3 & Acc-5 & MAE $\downarrow$ & Corr   \\
\midrule
MISA     & 78.19 & 77.22 & 63.38 & 40.55 & 0.449 & 0.576 \\
Self-MM  & 78.26 & 78.00 & 64.92 & 40.77 & 0.421 & 0.584 \\
MMIM     & 75.42 & 73.10 & 60.69 & 37.42 & 0.475 & 0.528 \\
TETFN    & 80.23 & 79.25 & \textbf{65.86} & 41.94 & 0.424 & 0.589 \\
TFR-Net  & 69.15 & 58.44 & 54.12 & 33.85 & 0.562 & 0.254 \\
P-RMF    & 78.34 & 79.69 & 60.61 & 38.95 & 0.441 & 0.550 \\
TF-Mamba & 79.65 & 78.92 & 61.93 & 37.86 & 0.441 & 0.548 \\
SemMSA   & \textbf{80.62} & \textbf{80.61} & 65.15 & \textbf{42.83} & \textbf{0.408} & \textbf{0.640} \\
\midrule

\multicolumn{7}{c}{\textbf{Random Missing Rate $r=0.1$}} \\
\midrule
MISA     & 77.39 & 75.82 & 63.02 & 38.88 & 0.461 & 0.561 \\
Self-MM  & 77.32 & 76.76 & 63.53 & 40.26 & 0.433 & 0.563 \\
MMIM     & 74.25 & 72.08 & 60.90 & 37.27 & 0.473 & 0.529 \\
TETFN    & 78.92 & 77.70 & \textbf{64.62} & 41.36 & 0.432 & 0.578 \\
TFR-Net  & 68.85 & 59.38 & 53.25 & 30.12 & 0.596 & 0.203 \\
P-RMF    & 77.24 & 78.88 & 59.52 & 37.72 & 0.454 & 0.530 \\
TF-Mamba & \textbf{79.65} & 78.79 & 62.36 & 36.98 & 0.445 & 0.550 \\
SemMSA   & 79.54 & \textbf{79.83} & 63.39 & \textbf{43.05} & \textbf{0.425} & \textbf{0.625} \\
\midrule

\multicolumn{7}{c}{\textbf{Random Missing Rate $r=0.2$}} \\
\midrule
MISA     & 74.33 & 71.70 & 59.23 & 38.15 & 0.489 & 0.490 \\
Self-MM  & 74.98 & 73.71 & 61.71 & 38.37 & 0.464 & 0.500 \\
MMIM     & 72.36 & 69.80 & 57.33 & 37.27 & 0.504 & 0.460 \\
TETFN    & 75.49 & 73.59 & 61.56 & 39.46 & 0.457 & 0.527 \\
TFR-Net  & 68.64 & 59.74 & 53.61 & 29.03 & 0.619 & 0.191 \\
P-RMF    & 76.23 & 77.46 & 59.30 & 37.05 & 0.460 & 0.512 \\
TF-Mamba & \textbf{78.77} & 77.88 & 61.05 & 38.29 & 0.459 & 0.507 \\
SemMSA   & 77.78 & \textbf{78.36} & \textbf{62.52} & \textbf{42.16} & \textbf{0.441} & \textbf{0.590} \\
\midrule

\multicolumn{7}{c}{\textbf{Random Missing Rate $r=0.3$}} \\
\midrule
MISA     & 74.11 & 70.40 & 59.30 & 36.40 & 0.505 & 0.464 \\
Self-MM  & 74.76 & 72.85 & 59.81 & 37.93 & 0.474 & 0.487 \\
MMIM     & 72.36 & 69.52 & 58.06 & 37.71 & 0.512 & 0.436 \\
TETFN    & 75.86 & 73.28 & \textbf{61.92} & 38.80 & 0.463 & 0.521 \\
TFR-Net  & 68.42 & 59.88 & 52.30 & 27.64 & 0.640 & 0.182 \\
P-RMF    & 75.66 & 75.89 & 56.89 & 36.89 & 0.475 & 0.483 \\
TF-Mamba & 75.93 & 74.72 & 58.42 & 39.82 & 0.468 & 0.485 \\
SemMSA   & \textbf{78.00} & \textbf{78.77} & 61.21 & \textbf{41.29} & \textbf{0.456} & \textbf{0.555} \\
\midrule

\multicolumn{7}{c}{\textbf{Random Missing Rate $r=0.4$}} \\
\midrule
MISA     & 72.87 & 67.52 & 57.33 & 34.86 & 0.523 & 0.436 \\
Self-MM  & 73.30 & 70.36 & 58.28 & 34.57 & 0.482 & 0.479 \\
MMIM     & 69.95 & 66.49 & 55.36 & 34.57 & 0.533 & 0.399 \\
TETFN    & 73.81 & 70.66 & 58.93 & 35.81 & 0.473 & 0.504 \\
TFR-Net  & 67.91 & 59.16 & 51.86 & 25.31 & 0.664 & 0.176 \\
P-RMF    & 75.32 & 76.30 & 55.80 & 36.54 & 0.482 & 0.472 \\
TF-Mamba & 75.49 & 73.72 & 60.39 & \textbf{40.04} & \textbf{0.470} & 0.477 \\
SemMSA   & \textbf{75.81} & \textbf{77.53} & \textbf{60.98} & 38.44 & 0.473 & \textbf{0.527} \\
\bottomrule
\end{tabular}%
}
\end{minipage}%
\hfill%
\begin{minipage}[t]{0.49\textwidth}
\centering
\resizebox{\linewidth}{!}{%
\begin{tabular}{lcccccc}
\toprule
\multicolumn{7}{c}{\textbf{Random Missing Rate $r=0.5$}} \\
\midrule
Method & Acc-2 & F1 & Acc-3 & Acc-5 & MAE $\downarrow$ & Corr   \\
\midrule
MISA     & 71.26 & 64.16 & 54.78 & 30.56 & 0.552 & 0.367 \\
Self-MM  & 71.41 & 67.11 & 53.90 & 32.02 & 0.517 & 0.390 \\
MMIM     & 68.49 & 64.81 & 52.37 & 33.41 & 0.553 & 0.336 \\
TETFN    & 72.43 & 67.30 & 56.24 & 33.48 & 0.512 & 0.394 \\
TFR-Net  & 67.47 & 58.66 & 52.37 & 24.65 & 0.685 & 0.171 \\
P-RMF    & 73.61 & 74.83 & 54.83 & 34.79 & 0.495 & 0.404 \\
TF-Mamba & \textbf{75.71} & 73.61 & \textbf{58.64} & 37.20 & 0.495 & 0.424 \\
SemMSA   & 73.84 & \textbf{76.27} & 56.60 & \textbf{38.21} & \textbf{0.486} & \textbf{0.480} \\
\midrule

\multicolumn{7}{c}{\textbf{Random Missing Rate $r=0.6$}} \\
\midrule
MISA     & 70.46 & 61.81 & 53.97 & 27.72 & 0.578 & 0.286 \\
Self-MM  & 70.02 & 64.21 & 51.86 & 29.10 & 0.548 & 0.313 \\
MMIM     & 67.91 & 63.86 & 49.31 & 29.18 & 0.578 & 0.270 \\
TETFN    & 70.97 & 64.19 & 53.54 & 29.90 & 0.545 & 0.309 \\
TFR-Net  & 67.03 & 58.30 & 52.59 & 24.80 & 0.696 & 0.157 \\
P-RMF    & 72.12 & 73.32 & 53.05 & \textbf{33.92} & 0.515 & 0.383 \\
TF-Mamba & \textbf{73.09} & 69.44 & 54.92 & 32.82 & 0.523 & 0.370 \\
SemMSA   & 72.96 & \textbf{75.94} & \textbf{55.74} & 33.64 & \textbf{0.495} & \textbf{0.455} \\
\midrule

\multicolumn{7}{c}{\textbf{Random Missing Rate $r=0.7$}} \\
\midrule
MISA     & 69.95 & 59.54 & 52.52 & 24.87 & 0.601 & 0.167 \\
Self-MM  & 69.58 & 62.28 & 50.62 & 25.53 & 0.571 & 0.198 \\
MMIM     & 66.89 & 62.23 & 46.53 & 28.59 & 0.595 & 0.190 \\
TETFN    & 69.29 & 61.09 & 51.06 & 27.86 & 0.572 & 0.190 \\
TFR-Net  & 67.18 & 58.15 & 52.30 & 23.78 & 0.707 & 0.163 \\
P-RMF    & 71.58 & 72.22 & 52.52 & \textbf{32.23} & 0.531 & 0.369 \\
TF-Mamba & 71.77 & 67.46 & 48.80 & 28.45 & 0.570 & 0.248 \\
SemMSA   & \textbf{72.09} & \textbf{75.78} & \textbf{55.52} & 31.43 & \textbf{0.510} & \textbf{0.425} \\
\midrule

\multicolumn{7}{c}{\textbf{Random Missing Rate $r=0.8$}} \\
\midrule
MISA     & 69.37 & 57.82 & 52.22 & 22.69 & 0.610 & 0.092 \\
Self-MM  & 69.51 & 60.68 & 50.77 & 22.03 & 0.585 & 0.138 \\
MMIM     & 65.28 & 60.53 & 45.35 & 22.32 & 0.607 & 0.145 \\
TETFN    & 69.88 & 60.92 & 49.60 & 23.48 & 0.584 & 0.154 \\
TFR-Net  & 67.54 & 57.55 & 52.74 & 22.97 & 0.721 & 0.100 \\
P-RMF    & 69.51 & 70.35 & 49.89 & \textbf{31.07} & 0.568 & 0.269 \\
TF-Mamba & 71.55 & 66.58 & 46.39 & 26.48 & 0.607 & 0.139 \\
SemMSA   & \textbf{72.75} & \textbf{77.12} & \textbf{54.21} & 24.65 & \textbf{0.520} & \textbf{0.389} \\
\midrule

\multicolumn{7}{c}{\textbf{Random Missing Rate $r=0.9$}} \\
\midrule
MISA     & 69.22 & 57.01 & 52.95 & 20.64 & 0.617 & 0.041 \\
Self-MM  & 68.92 & 58.32 & 52.15 & 22.17 & 0.586 & 0.111 \\
MMIM     & 65.72 & 59.64 & 42.67 & 20.35 & 0.610 & 0.096 \\
TETFN    & 68.92 & 58.75 & 45.73 & 22.10 & 0.590 & 0.108 \\
TFR-Net  & 69.08 & 57.71 & \textbf{53.76} & 23.05 & 0.721 & 0.088 \\
P-RMF    & 66.77 & 67.54 & 45.08 & \textbf{29.10} & 0.581 & 0.168 \\
TF-Mamba & 65.21 & 60.87 & 42.23 & 26.70 & 0.648 & 0.114 \\
SemMSA   & \textbf{71.22} & \textbf{74.86} & 53.11 & 21.14 & \textbf{0.528} & \textbf{0.327} \\
\bottomrule
\end{tabular}%
}
\end{minipage}

\vspace{-2mm}
\end{table*}

% \newpage
On the SIMS dataset, SemMSA also demonstrates strong robustness, especially in terms of F1, MAE, and Corr, as shown in Table~\ref{table13}. 
At $r=0$, SemMSA achieves the best Acc-2, F1, Acc-5, MAE, and Corr, with 80.62 Acc-2, 80.61 F1, 42.83 Acc-5, 0.408 MAE, and 0.640 Corr. 
When the missing rate increases to $r=0.3$, SemMSA still achieves 78.00 Acc-2, 78.77 F1, 41.29 Acc-5, 0.456 MAE, and 0.555 Corr, surpassing competing methods on most metrics. 
At higher missing rates, SemMSA remains particularly strong in F1 and Corr. 
For example, at $r=0.8$, SemMSA achieves the best Acc-2, F1, Acc-3, MAE, and Corr, indicating that the proposed method can better preserve sentiment-discriminative information in challenging Chinese multimodal sentiment scenarios. 
At $r=0.9$, SemMSA still obtains the highest Acc-2, F1, MAE, and Corr, further confirming its robustness under severe intra-modal missingness.

The superior robustness of SemMSA can be attributed to two key designs. 
First, Cross-modal Semantic Refinement enriches incomplete multimodal inputs with compact latent sentiment semantics from the frozen LLM, which helps alleviate the information loss caused by missing visual, acoustic, or textual tokens. 
Second, Cross-modal Spectral Alignment jointly aligns semantic, visual, acoustic, and language representations through a kernel Gram matrix, enabling anchor-free modeling of global nonlinear cross-modal dependencies. 
By combining semantic compensation with spectral alignment, SemMSA learns representations that are both sentiment-aware and cross-modally coherent, leading to stable and robust performance across different datasets and missing rates.

\begin{figure*}[t]
  \centering
  \includegraphics[width=\linewidth]{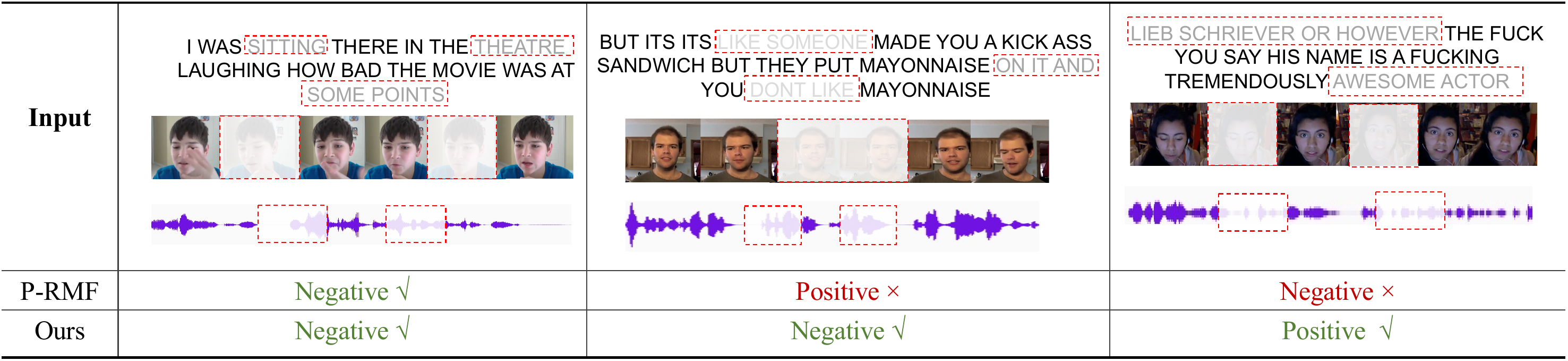} 
  \caption{Case study on the MOSI testing set, where SemMSA corrects challenging examples misclassified by P-RMF under incomplete multimodal inputs.}
  \label{fig5}
  \vspace{-0.15cm}
\end{figure*}

\subsubsection{Case Study}
To intuitively evaluate the robustness of SemMSA, we visualize three testing examples from the MOSI dataset in Figure~\ref{fig5}, where the red dashed boxes denote missing or corrupted language, visual, and acoustic regions. 
In the first case, both P-RMF~\cite{zhu2025proxy} and SemMSA correctly predict the negative sentiment, showing that the remaining multimodal cues are still sufficient for sentiment recognition. 
However, in the second case, P-RMF incorrectly predicts positive sentiment, possibly because it is misled by local expressions such as ``kick ass'' while failing to capture the overall negative meaning conveyed by ``you don't like mayonnaise'' under missing textual, facial, and acoustic cues. 
SemMSA correctly identifies the negative sentiment by leveraging latent sentiment semantics to compensate for incomplete evidence. 
These results suggest that Cross-modal Semantic Refinement helps recover sentiment-relevant semantics from incomplete inputs, while Cross-modal Spectral Alignment further enforces global consistency among language, visual, acoustic, and semantic representations, enabling accurate prediction under challenging intra-modal missingness.

\subsection{Limitations}
While SemMSA demonstrates strong performance across a variety of MSA benchmarks, several limitations remain. First, our experiments are conducted on widely used MSA benchmarks, including MOSI, MOSEI, and SIMS. 
These datasets mainly focus on sentiment analysis in video-based opinion scenarios. 
The generalization ability of SemMSA to broader multimodal affective understanding tasks, such as emotion recognition in conversation, multimodal sarcasm detection, or real-world long-form human interaction analysis, remains to be further explored. Second, SemMSA mainly exploits the hidden states of the frozen LLM as continuous semantic representations, but the semantic refinement process remains implicit. Since no explicit textual rationale is decoded, the generated latent semantic states are less interpretable than natural-language explanations. Finally, the current CSA module aligns semantic, visual, acoustic, and language representations through a kernel Gram matrix within each instance. 
Although the instance-level spectral separation loss helps preserve discriminability, the alignment quality may still be affected when most available modality cues are unreliable. 
In such cases, the refined semantics may not fully compensate for the missing evidence.

\subsection{Societal Impact}
SemMSA aims to improve robust multimodal sentiment analysis under incomplete data, which has potential benefits for real-world human-centered applications. 
By effectively leveraging language, visual, and acoustic cues even when some modalities are missing or corrupted, SemMSA can support more reliable affective understanding in scenarios such as online education, human-computer interaction, healthcare assistance, customer feedback analysis, and intelligent social platforms. 
Its semantic-aided design may help systems better interpret user attitudes and emotional states under noisy or imperfect sensing conditions, thereby improving accessibility and interaction quality. However, multimodal sentiment analysis also raises several societal concerns. 
First, sentiment prediction involves sensitive human-centered information, and the use of visual, acoustic, and language data may introduce privacy risks if deployed without appropriate consent, anonymization, and data protection mechanisms. 
Second, models trained on existing multimodal datasets may inherit demographic, cultural, linguistic, or contextual biases, potentially leading to unfair or inaccurate predictions for underrepresented groups. 
Third, incorrect sentiment predictions may negatively affect users if applied in high-stakes domains such as hiring and legal decision-making.

\section*{NeurIPS Paper Checklist}

%%% END INSTRUCTIONS %%%

\begin{enumerate}

\item {\bf Claims}
    \item[] Question: Do the main claims made in the abstract and introduction accurately reflect the paper's contributions and scope?
    \item[] Answer: \answerYes{} % Replace by \answerYes{}, \answerNo{}, or \answerNA{}.
    \item[] Justification: The key claims in the abstract and introduction accurately reflect the contribution and scope of the paper.

\item {\bf Limitations}
    \item[] Question: Does the paper discuss the limitations of the work performed by the authors?
    \item[] Answer: \answerYes{} % Replace by \answerYes{}, \answerNo{}, or \answerNA{}.
    \item[] Justification: The limitations of the work are discussed in the Limitations section of the Appendix.

\item {\bf Theory assumptions and proofs}
    \item[] Question: For each theoretical result, does the paper provide the full set of assumptions and a complete (and correct) proof?
    \item[] Answer: \answerYes{} % Replace by \answerYes{}, \answerNo{}, or \answerNA{}.
    \item[] Justification: The assumptions and complete proofs of the theoretical results on CSA are provided in the Theoretical Analysis section of the Appendix.

    \item {\bf Experimental result reproducibility}
    \item[] Question: Does the paper fully disclose all the information needed to reproduce the main experimental results of the paper to the extent that it affects the main claims and/or conclusions of the paper (regardless of whether the code and data are provided or not)?
    \item[] Answer: \answerYes{} % Replace by \answerYes{}, \answerNo{}, or \answerNA{}.
    \item[] Justification: Section 4.1 and Table~\ref{table9} give the implementation details and the hyperparameters of the configuration used for each metric, and these configurations are released with the code.

\item {\bf Open access to data and code}
    \item[] Question: Does the paper provide open access to the data and code, with sufficient instructions to faithfully reproduce the main experimental results, as described in supplemental material?
    \item[] Answer: \answerYes{} % Replace by \answerYes{}, \answerNo{}, or \answerNA{}.
    \item[] Justification: The code, data, and CH-SIMS checkpoints are released at \url{https://github.com/peacelwh/SemMSA}, \url{https://huggingface.co/datasets/Bluexxx/SemMSA-data}, and \url{https://huggingface.co/Bluexxx/SemMSA}, with instructions for training and evaluation.

\item {\bf Experimental setting/details}
    \item[] Question: Does the paper specify all the training and test details (e.g., data splits, hyperparameters, how they were chosen, type of optimizer, etc.) necessary to understand the results?
    \item[] Answer: \answerYes{} % Replace by \answerYes{}, \answerNo{}, or \answerNA{}.
    \item[] Justification: Data splits, hyperparameters, the optimizer, and checkpoint selection are given in Section 4.1 and the Appendix (Table~\ref{table9}).

\item {\bf Experiment statistical significance}
    \item[] Question: Does the paper report error bars suitably and correctly defined or other appropriate information about the statistical significance of the experiments?
    \item[] Answer: \answerNo{} % Replace by \answerYes{}, \answerNo{}, or \answerNA{}.
    \item[] Justification: Error bars are not reported; following prior works, results are obtained with random seeds 1111, 1112, and 1113.

\item {\bf Experiments compute resources}
    \item[] Question: For each experiment, does the paper provide sufficient information on the computer resources (type of compute workers, memory, time of execution) needed to reproduce the experiments?
    \item[] Answer: \answerYes{} % Replace by \answerYes{}, \answerNo{}, or \answerNA{}.
    \item[] Justification: Section 4.1 and the Appendix report the GPUs (NVIDIA RTX 6000 Ada and RTX 4090), and Table~\ref{table6} reports trainable parameters, GFLOPs, memory, and latency.
    
\item {\bf Code of ethics}
    \item[] Question: Does the research conducted in the paper conform, in every respect, with the NeurIPS Code of Ethics \url{https://neurips.cc/public/EthicsGuidelines}?
    \item[] Answer: \answerYes{} % Replace by \answerYes{}, \answerNo{}, or \answerNA{}.
    \item[] Justification: The research conducted in the paper conforms, in every respect, with the NeurIPS Code of Ethics.

\item {\bf Broader impacts}
    \item[] Question: Does the paper discuss both potential positive societal impacts and negative societal impacts of the work performed?
    \item[] Answer: \answerYes{} % Replace by \answerYes{}, \answerNo{}, or \answerNA{}.
    \item[] Justification: Both potential positive and negative societal impacts are discussed in the Societal Impact section of the Appendix.
    
\item {\bf Safeguards}
    \item[] Question: Does the paper describe safeguards that have been put in place for responsible release of data or models that have a high risk for misuse (e.g., pretrained language models, image generators, or scraped datasets)?
    \item[] Answer: \answerNA{} % Replace by \answerYes{}, \answerNo{}, or \answerNA{}.
    \item[] Justification: The paper poses no such risks.

\item {\bf Licenses for existing assets}
    \item[] Question: Are the creators or original owners of assets (e.g., code, data, models), used in the paper, properly credited and are the license and terms of use explicitly mentioned and properly respected?
    \item[] Answer: \answerYes{} % Replace by \answerYes{}, \answerNo{}, or \answerNA{}.
    \item[] Justification: The datasets (CMU-MOSI, CMU-MOSEI, and CH-SIMS), BERT, and Qwen3-1.7B are cited, and their licenses and terms of use are respected; the released features follow the terms of the original datasets.

\item {\bf New assets}
    \item[] Question: Are new assets introduced in the paper well documented and is the documentation provided alongside the assets?
    \item[] Answer: \answerYes{} % Replace by \answerYes{}, \answerNo{}, or \answerNA{}.
    \item[] Justification: The code (MIT License), the processed features, and the CH-SIMS checkpoints are documented in the README of the repository and on the Hugging Face dataset and model pages.

\item {\bf Crowdsourcing and research with human subjects}
    \item[] Question: For crowdsourcing experiments and research with human subjects, does the paper include the full text of instructions given to participants and screenshots, if applicable, as well as details about compensation (if any)? 
    \item[] Answer: \answerNA{} % Replace by \answerYes{}, \answerNo{}, or \answerNA{}.
    \item[] Justification: The paper does not involve crowdsourcing nor research with human subjects.

\item {\bf Institutional review board (IRB) approvals or equivalent for research with human subjects}
    \item[] Question: Does the paper describe potential risks incurred by study participants, whether such risks were disclosed to the subjects, and whether Institutional Review Board (IRB) approvals (or an equivalent approval/review based on the requirements of your country or institution) were obtained?
    \item[] Answer: \answerNA{} % Replace by \answerYes{}, \answerNo{}, or \answerNA{}.
    \item[] Justification: The paper does not involve crowdsourcing nor research with human subjects.

\item {\bf Declaration of LLM usage}
    \item[] Question: Does the paper describe the usage of LLMs if it is an important, original, or non-standard component of the core methods in this research? Note that if the LLM is used only for writing, editing, or formatting purposes and does not impact the core methodology, scientific rigorousness, or originality of the research, declaration is not required.
    %this research? 
    \item[] Answer: \answerYes{} % Replace by \answerYes{}, \answerNo{}, or \answerNA{}.
    \item[] Justification:  The LLM usage is deeply involved in the design of core methodology of the paper and carefully described in the manuscript and appendix.

\end{enumerate}

\end{document}